\documentclass{article}
\usepackage[T1]{fontenc}
\usepackage{iclr2027_conference,times}

\usepackage{amsmath,amsfonts,bm}

\def\eqref#1{equation~\ref{#1}}

\def\1{\bm{1}}

\DeclareMathAlphabet{\mathsfit}{\encodingdefault}{\sfdefault}{m}{sl}
\SetMathAlphabet{\mathsfit}{bold}{\encodingdefault}{\sfdefault}{bx}{n}

\usepackage{graphicx,booktabs,amssymb,amsthm,xcolor,array}
\usepackage{hyperref,url}
\usepackage{etoc}
\usepackage{placeins}
\AtBeginDocument{\etocdepthtag.toc{mtchapter}}
\newcolumntype{P}[1]{>{\raggedright\arraybackslash}p{#1}}
\newtheorem{proposition}{Proposition}
\newtheorem{lemma}{Lemma}

\newif\ifshowrwcomments
\showrwcommentsfalse

\title{What Must a World Model Distinguish for Planning?}
\author{%
Rongzhe Wei\textsuperscript{1}\quad
Hans Hao-Hsun Hsu\textsuperscript{1}\quad
Peizhi Niu\textsuperscript{2}\quad
Yifan Li\textsuperscript{3}\quad
Pan Li\textsuperscript{1}\\[3pt]
{\normalfont\small \textsuperscript{1}Georgia Institute of Technology}\\
{\normalfont\small \textsuperscript{2}University of Illinois Urbana-Champaign\quad
\textsuperscript{3}Tsinghua University}\\[3pt]
{\normalfont\small\texttt{\{rongzhe.wei,hans.hsu,panli\}@gatech.edu}}\\
{\normalfont\small\texttt{peizhin2@illinois.edu}\quad
\texttt{lyf21@mails.tsinghua.edu.cn}}}
\iclrfinalcopy
\hypersetup{
  pdftitle={What Must a World Model Distinguish for Planning?},
  pdfauthor={Rongzhe Wei, Hans Hao-Hsun Hsu, Peizhi Niu, Yifan Li, Pan Li}
}

\begin{document}
\maketitle
\fancyhead{}
\renewcommand{\headrulewidth}{0pt}
\begin{abstract}
World models simulate the consequences of action candidates, but good planning need not preserve every physical distinction required for accurate prediction. We formalize this gap through a hierarchy of mechanism, response, and decision sufficiency. Given a candidate set, the planning query determines which physical variations matter and how precisely they must be preserved: coarse decisions can discard much of the information needed for prediction, whereas fine decisions may require nearly the same resolution. In practice, planners often adaptively search to construct candidates, and information unnecessary for final selection may still be needed to discover good candidates. What a world model must preserve therefore depends on the query, the candidate set, and the planner. We study these effects in a collision system, nonlinear dynamics, and robotic planning. These varying requirements raise a design question: where should query information enter the planning system? A model that jointly generates actions and outcomes conditioned on the query achieves lower regret than an action-conditioned world model on seen objectives, but this advantage largely disappears when generalizing to unseen objectives. Motivated by this, we propose a modular design in which the query determines where to look and an action-conditioned model predicts what will happen, allowing the same predictions to be reused across objectives.

\end{abstract}
\vspace{-4mm}
\section{Introduction}
\label{sec:intro}
\vspace{-2mm}
World models support planning by predicting the consequences of candidate actions~\citep{ha2018world,hafner2019learning}. Recent systems extend this approach to continuous control and visual robotic planning~\citep{hansen2024td,zhou2024dino,assran2025v}. Yet not every predictable distinction between physical worlds matters for action choice~\citep{grimm2020value}. For instance, two worlds may produce different trajectories yet favor the same action, while a small change in another physical property can reverse the preferred action. This raises a question beyond predicting what will happen: \emph{which physical distinctions must a world model preserve for planning, and at what resolution?}

A planning query specifies how candidate outcomes are evaluated. Given a candidate set, we distinguish three progressively weaker notions of sufficiency: preserving distinctions in physical dynamics (\emph{mechanism}), preserving the costs of candidate actions under the query (\emph{response}), and preserving enough information to choose a low-regret action (\emph{decision}). Moving from mechanism to response asks \emph{which} physical variations affect candidate costs, while moving from response to decision asks \emph{how precisely} those variations must be resolved. The resulting response--decision gap is \emph{query dependent}: even for the same physical quantity, coarse decisions can discard substantial response detail, whereas fine decisions may require nearly the same resolution as predicting candidate costs.

However, the decision itself is not sufficient to support the planner. In practice, candidate actions are often constructed during planning rather than fixed in advance~\citep{zhou2024dino,assran2025v}. Iterative planners use model predictions to determine which parts of the action space to explore and ultimately which action to select. We find that a representation sufficient for final selection can discard distinctions needed to discover useful candidates. Richer information is especially valuable early in search. What a world model must preserve therefore depends not only on the current candidate set, but also on how the planner uses its predictions, for example to update the search distribution or to make the final choice. The query can further guide which actions are considered. In our controlled experiment, giving the proposal network access to the query achieves higher success with \(16\times\) fewer candidate evaluations than generating candidates without it. These constraints also shape how the representation should be learned. When capacity is insufficient to retain all predictive detail, supervision tied more directly to the planner's decisions better preserves information useful for planning, although this benefit shrinks as capacity increases. Together, these findings raise a broader design question: \emph{where should query information enter the planning system?}

We evaluate two paradigms with distinct query placements. In the action-conditioned world model (ACWM) paradigm~\citep{zhang2026world}, outcome prediction is unconditioned on the query, allowing the same prediction to be evaluated under different objectives. In contrast, the world-action model (WAM) paradigm~\citep{ye2026world,team2026motubrain} conditions both action generation and outcome prediction on the query. This distinction becomes important when objectives change: a new objective changes how outcomes are evaluated, but not the physical consequence of a given action. Empirically, WAM achieves lower regret on seen objectives, yet this advantage largely disappears on unseen objective combinations. These findings motivate a modular design that separates search from physical prediction: the query guides candidate action proposals, while an action-conditioned model predicts their outcomes. This design achieves lower held-out regret than ACWM and lower seen-objective regret than WAM, allowing search to adapt across objectives without coupling physical prediction to specific objectives.

\vspace{-2mm}
\section{Related Work}
\label{sec:related_work}
\vspace{-2mm}

\textbf{World models for planning.}
Learned world models support control through latent dynamics and imagined rollouts~\citep{ha2018world,watter2015embed,hafner2019learning,hafner2019dream,hafner2025mastering}, with applications to continuous control and visual planning~\citep{hansen2022temporal,hansen2024td,zhou2024dino,assran2025v,terver2025drives}. Within these systems, the planning objective can guide search over action-conditioned predictions~\citep{zhang2026world} or condition the joint generation of actions and outcomes~\citep{ye2026world,team2026motubrain}. SAGE separates these roles by using goal-conditioned subgoals to guide candidate generation while retaining a frozen world model for evaluation~\citep{cheng2026sage}. Our work studies how the information needed to find useful candidates differs from that needed to choose among them, and how query placement affects generalization to unseen objective compositions.

\textbf{Task-relevant abstractions for planning.}
State abstraction, bisimulation, and value-aware modeling characterize which distinctions can be discarded while preserving specified reward, value, or policy properties~\citep{li2006towards,ferns2012metrics,farahmand2017value,grimm2020value}. Related approaches learn representations and planning models without reconstructing observations~\citep{silver2017predictron,oh2017value,schrittwieser2020mastering,zhang2020learning}, while rate-distortion formulations relate model compression to bounded suboptimality~\citep{arumugam2022deciding}. More recently, query-conditioned physical modeling proposes adapting the physical abstraction and prediction fidelity to the intervention query~\citep{thorpe2026physically}. Building on this perspective, we distinguish the information required to predict candidate costs from the information required to make a low-regret decision. This separation reveals two distinct query roles. It determines which physical variations are relevant to planning and how precisely those variations must be resolved.

\textbf{Decision-aware representation learning.}
Information bottlenecks study compression that retains information relevant to downstream tasks~\citep{alemi2016deep,islam2022representation}, while goal-aware prediction directs dynamics modeling toward quantities relevant to a specified goal~\citep{nair2020goal}. Objective-mismatch analysis and decision-focused learning further study how training objectives align with downstream decision quality~\citep{lambert2020objective,mandi2024decision}. Our work decouples the representation needed for planning from the training objective used to induce it, examining when decision-focused supervision outperforms learning from candidate cost predictions.

\vspace{-2mm}
\section{Physical Distinguishability for World-Model Planning}
\label{sec:formulation}

\vspace{-2mm}
\subsection{World Models and Planning Queries}
\vspace{-1mm}
\label{sec:setup}

We consider physical worlds indexed by a mechanism $\theta\in\Theta$. Given a history of observations and actions $h_t=(o_{\leq t},a_{<t})$, a world model with parameters $\psi$ forms a predictive state $z_t=\mathcal{W}_\psi(h_t)$ and predicts an outcome $\widehat{Y}_\psi(A\mid z_t)$ for each candidate action sequence $A$. We write $Y(A;\theta)$ for the corresponding physical outcome, such as a trajectory, terminal pose, or contact outcome. For each candidate action, we compare physical worlds under the same initial conditions. 

A planning query $q$ specifies an evaluation function $\ell_q(Y,A)$ that assigns a scalar cost to outcome $Y$ and action $A$. The true candidate cost is $J_q(A;\theta)=\ell_q(Y(A;\theta),A)$. The model cost is $\widehat{J}_q(A\mid z_t)=\ell_q(\widehat{Y}_\psi(A\mid z_t),A)$. Assuming the minimum over the candidate set $C$ is attained, let $A_q^\star(\theta;C)$ denote the minimizer under a fixed tie-breaking rule: $A_q^\star(\theta;C)\in\arg\min_{A\in C}J_q(A;\theta)$.
We ask which distinctions between physical worlds $\theta,\theta'$ a predictive state must preserve for this choice.

\vspace{-2mm}
\subsection{A Hierarchy of World-Model Sufficiency}
\vspace{-1mm}
\label{sec:hierarchy}

For fixed $q$ and $C$, we introduce three progressively weaker sufficiency requirements through equivalence relations over physical worlds.

\textbf{Mechanism sufficiency.} We write $\theta\sim_M\theta'$ when two worlds induce the same physical dynamics. A mechanism-sufficient world model distinguishes worlds whenever their dynamics differ, regardless of their relevance to the query.

\textbf{Response sufficiency.} This requirement preserves the query-specific candidate-cost profile $\mathbf{J}_q^C(\theta)=(J_q(A;\theta))_{A\in C}$. We define
$\theta\sim_R\theta'$ if and only if $\mathbf{J}_q^C(\theta)=\mathbf{J}_q^C(\theta')$.
A response-sufficient world model may therefore merge physically different worlds that induce the same candidate-cost profile.

\textbf{Decision sufficiency.} Distinct candidate-cost profiles may still yield the same selected action. We define $\theta\sim_D\theta'$ if and only if $A_q^\star(\theta;C)=A_q^\star(\theta';C)$.

The equivalence relations above induce partitions
$\Pi_M$, $\Pi_R^{C,q}$, and $\Pi_D^{C,q}$ of $\Theta$, whose cells are the corresponding equivalence classes. These partitions satisfy $\Pi_M\preceq\Pi_R^{C,q}\preceq\Pi_D^{C,q}$, where $\Pi\preceq\Pi'$ means that every cell of $\Pi$ lies within a cell of $\Pi'$. A world model is $\Pi$-sufficient if any two worlds mapped to the same predictive state belong to the same $\Pi$-cell. These are information requirements, not guarantees of attainability: if the history cannot distinguish worlds in different cells, neither can a representation derived from it. Sec.~\ref{sec:query_structure} introduces tolerance-aware covers to extend the analysis to approximate response prediction and low-regret decisions.

\vspace{-2mm}
\section{Queries Determine What a World Model Must Preserve}
\label{sec:query}
\vspace{-2mm}

In this paper, we study how planning queries change the information a world model must preserve, first for a fixed candidate set and then when candidates are constructed by the planner. Fig.~\ref{fig:experimental_overview} summarizes the experimental settings used throughout Secs.~\ref{sec:query}--\ref{sec:query_design}, detailed architectures and training protocols are deferred to the appendix.

\vspace{-2mm}
\subsection{Query-Dependent Relevance and Resolution}
\vspace{-1mm}
\label{sec:query_structure}
Sec.~\ref{sec:hierarchy} separates two ways the query shapes what a world model must preserve. Moving from mechanism to response determines \emph{which} physical variations affect candidate costs, while response-to-decision determines \emph{how precisely} they must be resolved for low-regret planning.

For fixed $q$ and $C$, physically different worlds need not remain distinguishable if they induce the same candidate-cost profile $\mathbf{J}_q^C(\theta)$. Because queries evaluate physical outcomes differently, different queries can make different variations of the same physical system relevant. Even when two queries depend on the same physical variation, they may require different precision. 
To describe approximate decisions, we define the candidate regret
\begin{equation}
    \operatorname{Reg}_q(A;\theta,C)=J_q(A;\theta)-\min_{A'\in C} J_q(A';\theta).
\end{equation}
Worlds with different optimal actions can still admit the same low-regret action. Since this relation need not be transitive, approximate decision requirements are more naturally described by covers than by the exact partitions in Sec.~\ref{sec:hierarchy}. Let $\epsilon$ bound the maximum absolute error across candidate costs, $\|\widehat{\mathbf{J}}-\mathbf{J}\|_\infty\leq\epsilon$, and let $\Delta$ bound decision regret. The response cover size $N_R^{C,q}(\epsilon)$ is the minimum number of radius-$\epsilon$ sets needed to cover the candidate-cost profiles $\{\mathbf{J}_q^C(\theta):\theta\in\Theta\}$. The decision cover size $N_D^{C,q}(\Delta)$ is the minimum number of sets covering $\Theta$ such that each set admits a common action $A\in C$ with $\operatorname{Reg}_q(A;\theta,C)\leq\Delta$ for every world it contains. Uniformly accurate candidate costs connect these two requirements.

\begin{lemma}[Response-to-decision cover bound]
\label{lem:response_decision}
If $\sup_{A\in C}|\widehat{J}_q(A)-J_q(A;\theta)|\leq\epsilon$ and $\widehat{A}\in\arg\min_{A\in C}\widehat{J}_q(A)$, then
\begin{equation}
\operatorname{Reg}_q(\widehat{A};\theta,C)\leq 2\epsilon,
\qquad
N_D^{C,q}(2\epsilon)\leq N_R^{C,q}(\epsilon).
\end{equation}
\end{lemma}

The lemma gives a matched tolerance $\Delta=2\epsilon$: information sufficient to predict every candidate cost within $\epsilon$ is also sufficient to select an action with regret at most $2\epsilon$. The converse need not hold, as different candidate-cost profiles may still admit the same low-regret action. We therefore define the response--decision log-cover gap, $\log_2 N_R^{C,q}(\epsilon)-\log_2 N_D^{C,q}(2\epsilon)$, measured in bits. It quantifies how much less distinction may be required for decision than for response prediction. Sec.~\ref{sec:collision} shows that the query determines whether this gap grows or remains bounded as the required precision increases.

\begin{figure*}[t]
    \centering
    \vspace{-2mm}
    \includegraphics[width=\textwidth]{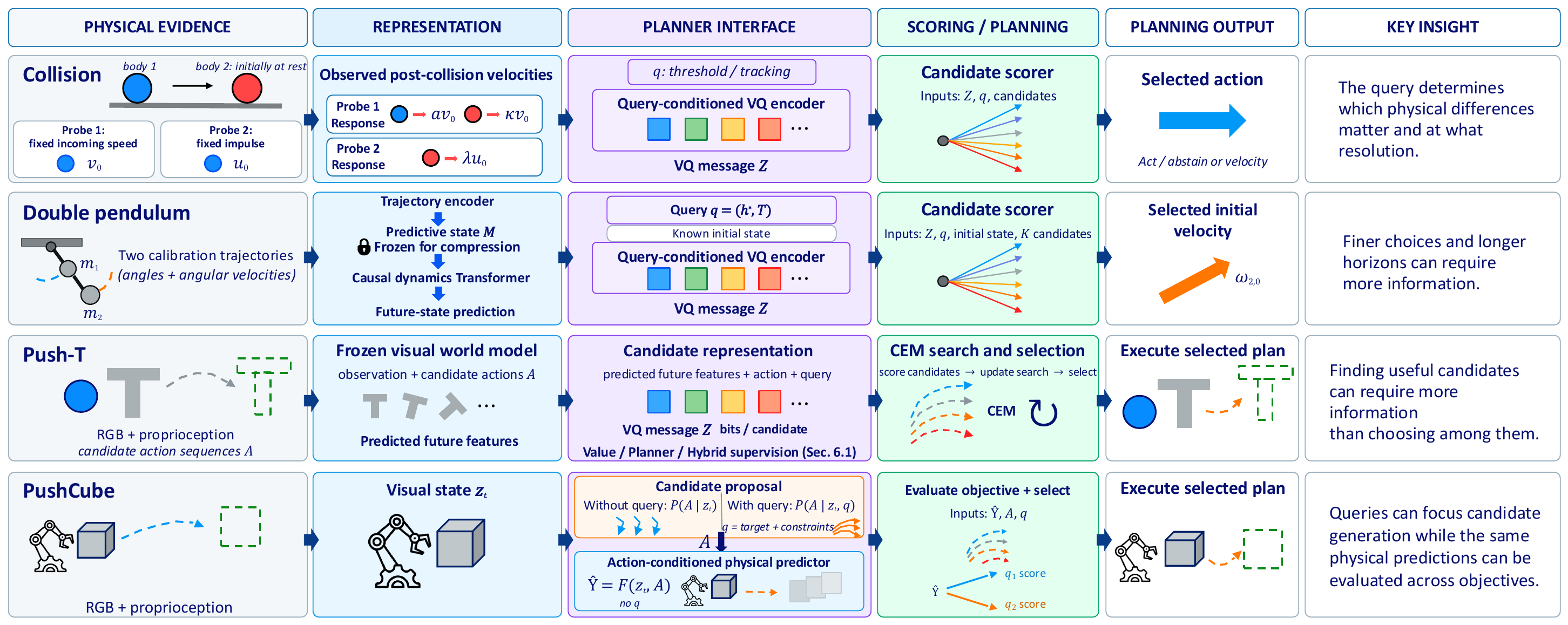}
    \vspace{-8mm}
    \caption{Overview of experimental settings.
    }
    \label{fig:experimental_overview}
    \vspace{-4mm}
\end{figure*}

\vspace{-1mm}
\subsection{Controlled Physics: Exact Query-Dependent Structure}
\vspace{-1mm}
\label{sec:collision}

We first use a one-dimensional collision system to separate the two effects in Sec.~\ref{sec:query_structure}.

\textit{\textbf{Case Study: Collision System.}}
As summarized in Fig.~\ref{fig:experimental_overview}, the mechanism is $\theta=(m_1,m_2,e)$, with body 2 initially at rest. Its post-collision velocity is $\kappa v$ for incoming velocity $v$, or $\lambda u$ for an impulse $u$, where $\kappa=(1+e)m_1/(m_1+m_2)$ and $\lambda=(1+e)/(m_1+m_2)$. We consider three queries: $q_1$ decides between abstaining (cost $0$) and a collision at fixed speed $v_0$ with cost $c_1-\kappa v_0$; $q_2$ makes the analogous binary choice for a fixed impulse $u_0$ with cost $c_2-\lambda u_0$; and $q_3$ selects $v$ to minimize velocity tracking error $|\kappa v-v_{\mathrm{tar}}|$. Full derivations and parameter ranges are given in App.~\ref{app:collision-covers}.

Queries $q_1$ and $q_2$ depend on different physical quantities, $\kappa$ and $\lambda$, neither of which determines the other over the mechanism family. In contrast, $q_1$ and $q_3$ both depend on $\kappa$ but require different precision: $q_1$ only needs to determine which side of the threshold $\kappa=c_1/v_0$ the world lies, whereas the optimal action for $q_3$ varies continuously with $\kappa$.

The query determines both which physical variation matters and how finely it must be resolved. For continuous tracking, let compact interval $C \subset \mathbb{R}_{+}$ contain every optimum $v_{\mathrm{tar}}/\kappa$.

\begin{proposition}[Collision relevance and resolution separation]
\label{prop:collision}
For threshold queries $q_1$ and $q_2$ strictly inside the ranges of $\kappa$ and $\lambda$, respectively,
$N_D^{C,q_i}(0)=2$ while $N_R^{C,q_i}(\epsilon)=\Theta(1/\epsilon)$ as $\epsilon\to0$ for $i\in\{1,2\}$.
For continuous-action $q_3$, under the endpoint condition in App.~\ref{app:collision-covers},
    \begin{equation}
        N_R^{C,q_3}(\epsilon)=\Theta(1/\epsilon),
        \qquad
        N_D^{C,q_3}(\Delta)=\Theta(1/\Delta).
    \end{equation}
\end{proposition}

At the matched tolerance $\Delta=2\epsilon$, the response--decision log-cover gap grows without bound for the threshold queries but remains bounded for tracking (Fig.~\ref{fig:collision_structural_operational}(a)). Thus, even when two queries depend on the same physical quantity, one may require only a coarse decision boundary while the other requires increasingly precise information.

\begin{figure*}[t]
\centering
\vspace{-2mm}
\setlength{\parskip}{0pt}
\includegraphics[width=0.8\linewidth]{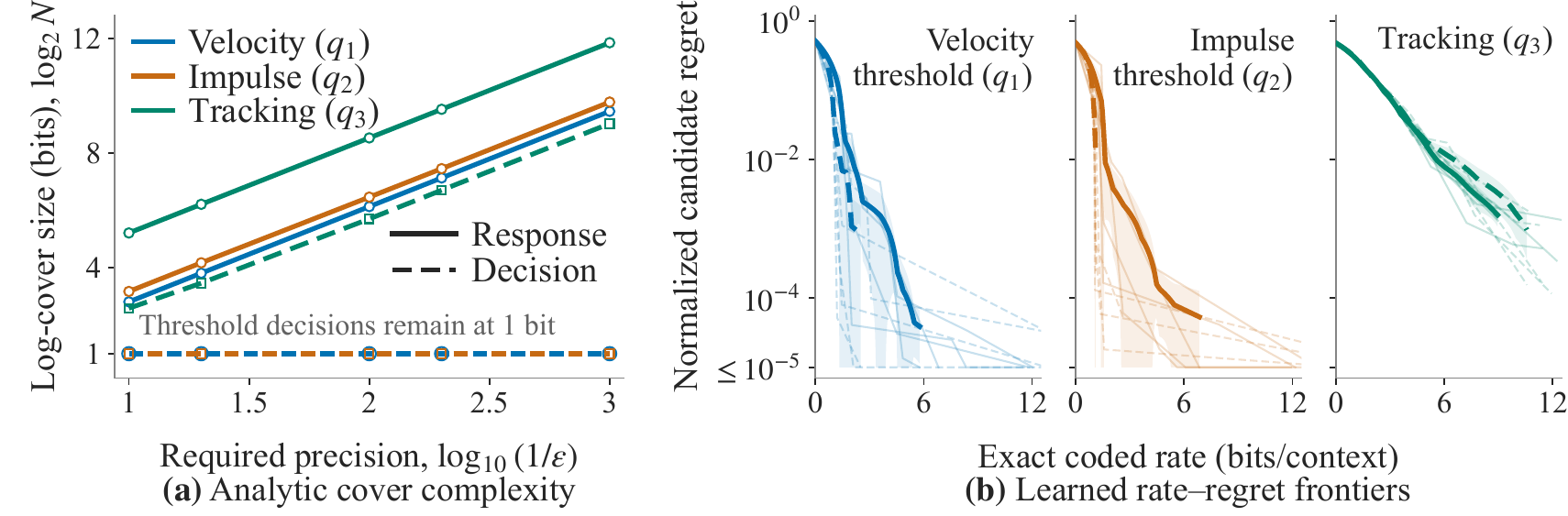}\par
\vspace{-2mm}
\small
\caption{\textbf{Query-dependent response--decision gaps.}
\textbf{(a)} Analytic cover complexity at matched tolerances $\Delta=2\epsilon$: the response--decision gap grows with precision for the threshold queries ($q_1, q_2$) but remains bounded for tracking ($q_3$).
\textbf{(b)} Learned rate--regret frontiers. Decision supervision requires fewer coded bits for the threshold queries, while no systematic saving appears for tracking.}
\label{fig:collision_structural_operational}
\vspace{-5mm}
\end{figure*}

\textbf{Structural gap in learned representations.} We next examine this structural gap in learned representations. Following Fig.~\ref{fig:experimental_overview}, an encoder maps query-independent collision probes and the query into a vector-quantized representation, whose coded rate directly quantifies the retained physical information. Using identical architectures, we train this representation under either \emph{response supervision} (predicting candidate costs) or \emph{decision supervision} (direct low-regret action selection). We then freeze each encoder and train a shared candidate scorer on top, ensuring that any performance difference reflects the information retained in the representation rather than head-specific effects. Full training and coding details appear in App.~\ref{app:collision-learning}.

\textbf{Decision supervision yields compression savings for coarse decisions, but not for fine tracking.}
In Fig.~\ref{fig:collision_structural_operational}(b), decision supervision requires fewer bits for both threshold queries across the reported tolerances, but provides no systematic saving for tracking. For example, at $\delta\approx10^{-3}$, $q_2$ requires $1.13$ bits per context under decision supervision, a saving of $1.61$ bits, whereas the saving for $q_3$ is $-1.02$ bits at $\delta\approx10^{-2}$. Lem.~\ref{lem:response_decision} orders the information required in principle, not the rates learned by different objectives. Decision sufficiency therefore does not imply that decision supervision is always the best way to learn the representation, a distinction revisited in Sec.~\ref{sec:selective_learning}.

\vspace{-2mm}
\subsection{Learned Temporal Dynamics: Resolution Shapes Information Needs}
\vspace{-1mm}
\label{sec:temporal}

We next study a nonlinear system in which a world model learns from trajectories, and examine how the information a planner needs from it varies with decision resolution and planning horizon.

\textit{\textbf{Case Study: Double-Pendulum Dynamics.}} We consider a planar double pendulum with hidden mechanism $\theta=(r_m,l_1,r_l)$, where $r_m=m_2/m_1$ is the mass ratio, $l_1$ is the first link length, and $r_l=l_2/l_1$ is the link-length ratio. The world model observes two trajectories of joint angles and angular velocities from fixed initial conditions shared across worlds, without access to $\theta$. Planning is performed from a new initial state: the planner observes the initial angles and first-joint angular velocity, and chooses the second-joint initial velocity $u=\omega_{2,0}$ from $K$ candidates. The query $q=(h^\star,T)$ specifies a target height $h^\star$ and planning horizon $T$, with cost determined by how closely the second mass's maximum height over the rollout matches $h^\star$. Larger $K$ provides finer action choices, while larger $T$ requires predicting the dynamics farther into the future. Full system and simulator details appear in App.~\ref{app:dp-wm}.

Following the pipeline in Fig.~\ref{fig:experimental_overview}, a Transformer encoder maps the observed trajectories to a query-independent physical representation, and a causal Transformer predicts future state transitions from it. We then freeze the representation, pass it through the VQ interface from Sec.~\ref{sec:collision}, and train a common scorer from simulator costs. Fixing the representation while varying $K$ or $T$ isolates the information required across different resolutions and horizons. 
As a fidelity check, the learned model substantially outperforms a mean-mechanism simulator, with median state RMSE $0.363$ versus $1.082$ at $T=4$\,s. Full architecture, training, and compression details appear in App.~\ref{app:dp-wm}--\ref{app:dp-controls}.

\begin{figure}[t]
\centering
\vspace{-2mm}
\setlength{\parskip}{0pt}
\includegraphics[width=\linewidth]{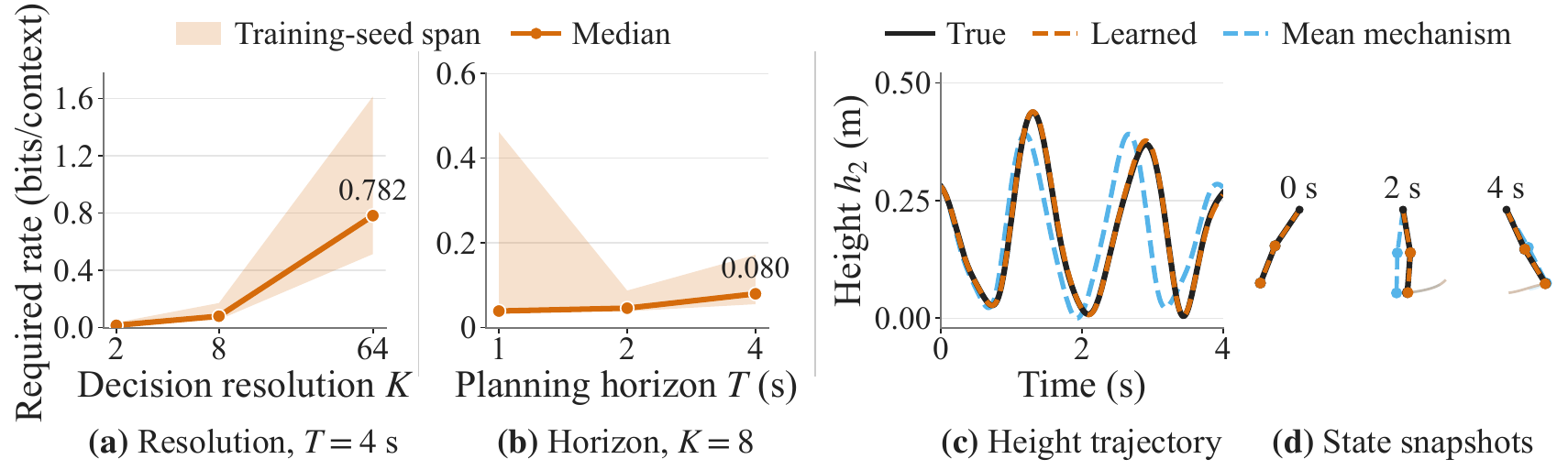}\par
\setlength{\abovecaptionskip}{4pt}
\setlength{\belowcaptionskip}{0pt}
\vspace{-2mm}
\small
\caption{\textbf{Resolution and horizon shape information needs.}
\textbf{(a,b)} Median required coded rate at normalized regret $\delta=0.02$ as candidate resolution and planning horizon vary. Shaded regions span the training seeds.
\textbf{(c,d)} Example height trajectories and corresponding state snapshots for the true dynamics, learned model, and mean-mechanism baseline.}
\label{fig:dp_resolution_horizon}
\vspace{-4mm}
\end{figure}

\textbf{Both finer action choices and longer planning horizons increase the physical information required for low-regret decisions.} At $T=4$ s and normalized regret $\delta=0.02$, increasing the candidate resolution from $K=2$ to $8$ and $64$ raises the median required rate from $0.016$ to $0.080$ and $0.782$ bits per context (Fig.~\ref{fig:dp_resolution_horizon}(a)). The same ordering holds at $\delta=0.05$ (App.~\ref{app:dp-rates}, Tab.~\ref{tab:dp-wm-resolution}). With $K=8$ and $\delta=0.02$, increasing the planning horizon from $T=1$ to $2$ and $4$ s raises the median required rate from $0.039$ to $0.046$ and $0.080$ bits per context (Fig.~\ref{fig:dp_resolution_horizon}(b)).

Together, these results extend the collision analysis to learned nonlinear dynamics: coarse decisions can retain substantially less physical information than fine ones, while longer horizons can require more information. Secs.~\ref{sec:collision} and \ref{sec:temporal} assume that candidate sets are given. Sec.~\ref{sec:planning} next considers how these requirements change when the planner must construct the candidate set itself.

\vspace{-2mm}
\section{Planning Changes What Must Be Distinguished}
\label{sec:planning}
\vspace{-2mm}

In adaptive planning, model predictions are used not only to select among candidates but also to construct them~\citep{m2023model}. We therefore ask how the required information changes with the candidate set and with the stage at which predictions enter the planner.

\vspace{-2mm}
\subsection{Candidate Sets Change Relevant Distinctions}
\vspace{-1mm}
\label{sec:support}

We first isolate how the candidate set changes the information needed for selection. With the query and selection rule fixed, the relevant distinctions can still change with the alternatives compared.

\textit{\textbf{Case Study: Push-T Planning.}}
In Push-T~\citep{chi2025diffusion}, a planar pusher moves and rotates a T-shaped object toward a target pose. A frozen visual world model~\citep{terver2025drives} predicts future features for candidate pushing sequences. The \emph{fine} query penalizes terminal pose error, while the \emph{coarse} query ignores errors within a fixed tolerance. The coarse query is used to study candidate-set dependence and information timing, while the fine query tests whether information sufficient for final selection also supports search. We compare full predicted futures with a \emph{Decision (1 bit)} representation trained for low-regret selection. Sec.~\ref{sec:search_selection} additionally uses two- and four-bit \emph{Cost} messages trained to preserve candidate costs. Full protocols are given in App.~\ref{app:pusht-search}.

\textbf{Broader candidate sets make selection more sensitive to information loss.}
Under the coarse query, the regret gap between the one-bit representation and full predictions widens from $0.019$ on the narrowest candidate sets to $0.115$ on the broadest (Fig.~\ref{fig:planning_stages}(a)). Thus, even with the query, candidate count, and selection rule fixed, information sufficient for one candidate set may no longer suffice as the alternatives change.

\begin{figure}[h]
\centering
\vspace{-3mm}
\includegraphics[width=0.85\linewidth]{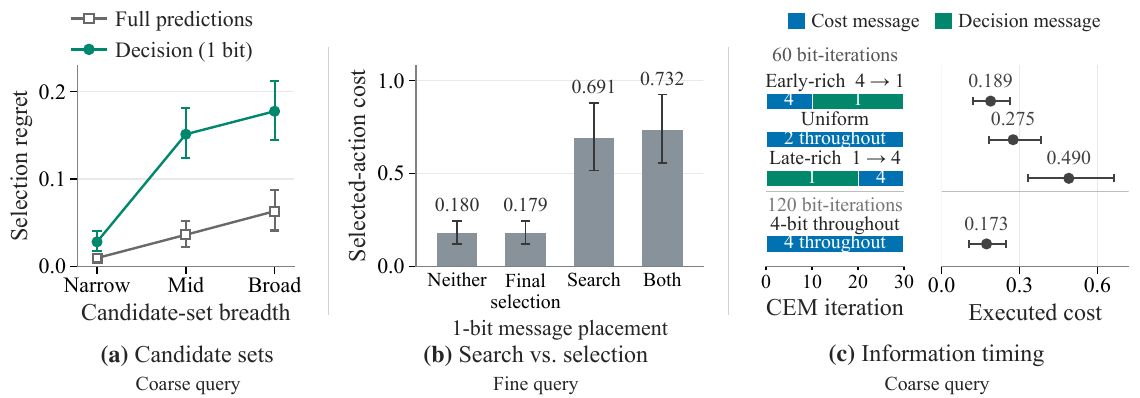}
\small
\vspace{-4mm}
\caption{\textbf{Candidate sets and planning stages shape information requirements.}
\textbf{(a)} Broader candidate sets amplify the gap between one-bit messages and full predictions.
\textbf{(b)} One-bit messages preserve final selection but substantially degrade search.
\textbf{(c)} Under a matched communication budget, richer feedback is more useful earlier in CEM search.}
\label{fig:planning_stages}
\vspace{-2mm}
\end{figure}

\vspace{-2mm}
\subsection{Different Planning Stages Require Different Distinctions}
\vspace{-1mm}
\label{sec:search_selection}

Adaptive search can require distinctions that final selection discards. For instance, in CEM~\citep{rubinstein1999cross}, each update depends only on the elite subset, whereas final selection depends only on the winning candidate.

\textbf{Information sufficient for final selection can fail during search.}
Under the fine query, we run CEM with either full predictions or one-bit Decision messages, then apply both selectors to each final 300-candidate population. Replacing full predictions only at final selection leaves mean selected-action cost nearly unchanged ($0.180$ vs.\ $0.179$). In contrast, using one-bit messages during search raises the cost to $0.691$ even when full predictions are restored for the final choice (Fig.~\ref{fig:planning_stages}(b)). Information adequate for selecting from a good candidate set can therefore fail to construct that set.

\textbf{Richer information is more valuable earlier in CEM search.}
Under the coarse query, we compare three 30-iteration schedules with the same budget of $60$ nominal bit-iterations per candidate: four-bit Cost messages early followed by one-bit Decision messages, the reverse schedule, and two-bit Cost messages throughout. Their mean executed costs are $0.189$, $0.490$, and $0.275$, respectively. Four-bit Cost messages throughout reach $0.173$ at twice the budget (Fig.~\ref{fig:planning_stages}(c)). Richer information is therefore most useful before the search distribution has narrowed, when early updates still determine which regions of the action space remain under consideration.

\vspace{-2mm}
\section{From Planning Requirements to Model Design}
\label{sec:query_design}
\vspace{-2mm}

Secs.~\ref{sec:query} and~\ref{sec:planning} identify the information planning needs, but not how a world model should preserve it or where the query should enter. We therefore ask two design questions: (1) When the representation must prioritize among predictive details, which supervision best preserves the distinctions useful for planning? (2) As objectives change, should the query affect candidate proposals, physical prediction, or both? We study these questions through matched supervision and interface comparisons. The results motivate a modular design in which the query guides candidate proposals while physical prediction remains action-conditioned.

\vspace{-2mm}
\subsection{Sufficiency Does Not Determine Supervision}
\vspace{-1mm}
\label{sec:selective_learning}

A sufficiency requirement specifies which distinctions a representation must preserve, but not which supervision will learn them most effectively. We study this question in Push-T using the frozen visual world model from Sec.~\ref{sec:support}. \emph{Value} trains the compact representation to preserve candidate costs, while \emph{Planner} targets candidate ordering and the elite-set boundary used by CEM. \emph{Hybrid} combines both signals. All three receive identical inputs and query information and use matched representation budgets, architectures, data, and optimization. Full details are given in App.~\ref{app:pusht-supervision}.

\begingroup
\setlength{\intextsep}{6pt}
\begin{figure}[!htb]
\centering
{\fontsize{8.3}{10}\selectfont
\renewcommand{\arraystretch}{0.65}
\begin{tabular}{@{}c@{\hspace{0.015151\linewidth}}c@{}}
\includegraphics[width=0.364393\linewidth]{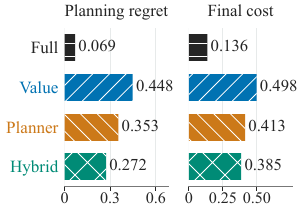} &
\includegraphics[width=0.620454\linewidth]{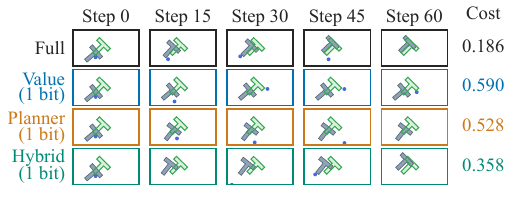} \\
(a) Quantitative summary & (b) Closed-loop execution
\end{tabular}\par}
\setlength{\abovecaptionskip}{4pt}
\setlength{\belowcaptionskip}{0pt}
\small
\caption{\textbf{Supervision shapes planning and closed-loop behavior.} \textbf{(a)} Push-T planning regret and mean closed-loop final cost (averaged over 128 episodes and 3 seeds). Full is uncompressed, others use 1-bit representations. \textbf{(b)} Representative rollouts under a common target query.}
\label{fig:supervision_results}
\end{figure}
\endgroup

\textbf{Planner-oriented supervision helps identify planning-relevant predictive details.} With one bit per candidate, mean planning regret decreases from $0.448$ for Value to $0.353$ for Planner and $0.272$ for Hybrid, compared with $0.069$ using the full predictions (Fig.~\ref{fig:supervision_results}(a)). The same ordering among the compact representations appears in closed-loop control: mean final cost is $0.498$, $0.413$, and $0.385$ for Value, Planner, and Hybrid, respectively, while Full reaches $0.136$. Fig.~\ref{fig:supervision_results}(b) shows a representative closed-loop execution under the same query target. This advantage weakens as representation capacity increases: when more predictive detail can be retained, preserving candidate costs becomes more competitive with planner-oriented supervision (App.~\ref{app:pusht-supervision}).

\vspace{-2mm}
\subsection{Where Should Query Information Enter?}
\vspace{-1mm}
\label{sec:selectivity_reuse}
\label{sec:query_search}

Sec.~\ref{sec:selective_learning} considered how to learn the information useful for a fixed planning query. We further ask where the query should enter when the objective itself changes. A new objective can change which actions are worth considering and how their outcomes are evaluated, while the physical consequence of a fixed action remains unchanged. We therefore separate two roles of the query: guiding which candidates the planner considers and conditioning the model that predicts their physical outcomes.

\textbf{Queries can focus candidate generation.}
\textit{\textbf{Case Study: Multi-Goal PushCube.}}
In ManiSkill~\citep{tao2024maniskill3}, a robot arm pushes a cube toward a query-specified target displacement. The target is not visible in the observation and is not provided to the physical predictor. We compare two proposal networks trained on the same physical trajectories. A query-aware proposal receives the target, while a query-blind proposal generates candidates from the current observation alone. Both use the same query-independent, action-conditioned world model to predict candidate outcomes, and the same query-based scorer selects the executed plan. The comparison therefore isolates whether the query is available during candidate generation. Full training and evaluation details are given in App.~\ref{app:query-search}.

\begin{figure}[h]
\centering
\vspace{-4mm}
\includegraphics[width=\linewidth]{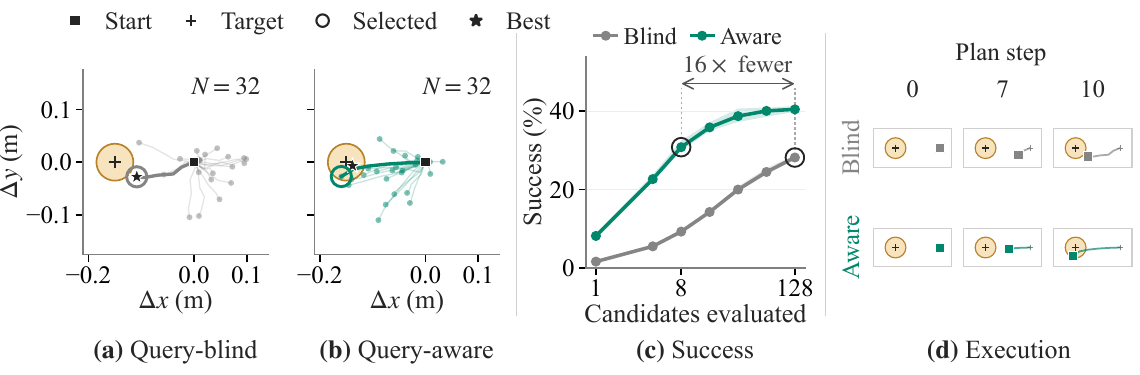}
\vspace{-8mm}
\small
\caption{\textbf{Query-guided proposals improve search efficiency.}
\textbf{(a,b)} Candidate outcomes from $32$ proposals for the same scene and target. Rings mark the candidates selected by the learned planner; stars mark the lowest-cost candidates in each set for evaluation only.
\textbf{(c)} Mean success across $512$ new scenes and four target directions as the candidate budget varies.
\textbf{(d)} Recorded cube positions along the selected plans in the illustrated scene.}
\label{fig:query_counterfactuals}
\vspace{-2mm}
\end{figure}

\textbf{Providing the query to the proposal yields useful candidates with far fewer evaluations.}
Query-blind proposals spread across multiple directions, whereas query-aware proposals concentrate more strongly toward the target (Fig.~\ref{fig:query_counterfactuals}(a,b)). Across 512 new scenes and four target directions, the query-aware proposal reaches $30.83\%$ success with only $N=8$ candidates, compared with $28.19\%$ for the query-blind proposal at $N=128$ (Fig.~\ref{fig:query_counterfactuals}(c)). Thus, query access achieves higher success with $16\times$ fewer candidate evaluations while leaving physical prediction query-independent. 

This result indicates that while the query effectively guides where the planner searches, it does not imply that physical predictions must also be query-conditioned. We examine this distinction next under shifting planning objectives.

\begin{figure}[h]
\centering
\vspace{-2mm}
\includegraphics[width=\linewidth]{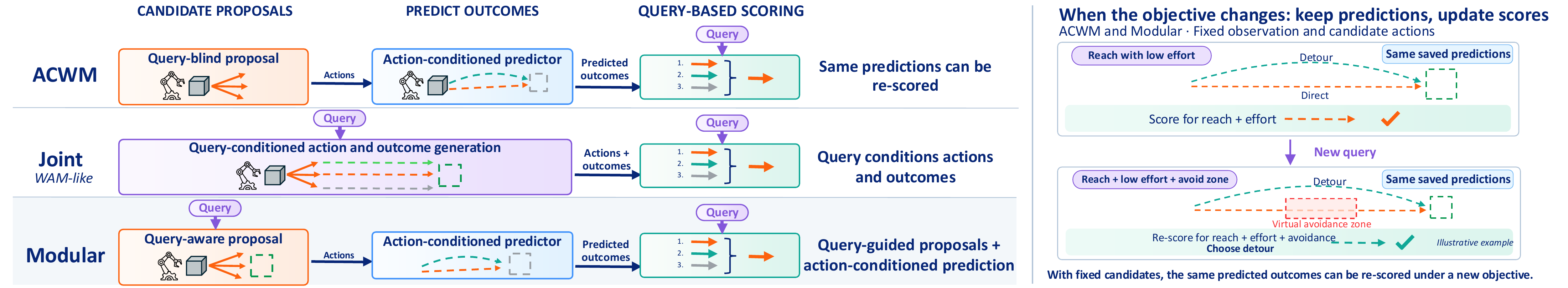}
\vspace{-6mm}
\small
\caption{\textbf{Query placement across planning interfaces.} ACWM is query-independent in both proposal and prediction; Joint conditions both on the query; and Modular guides proposals by the query while keeping physical prediction action-conditioned.}
\label{fig:query_placement}
\vspace{-2mm}
\end{figure}

\textit{\textbf{Case Study: Objective-Compositional PushCube.}}
We extend the PushCube setting above with objectives that combine target reaching with optional keep-out regions, a peak-speed constraint, and an effort penalty. These terms change how trajectories are evaluated without changing the underlying dynamics. Training uses five objective families, while evaluation includes five held-out compositions that recombine these terms or introduce an additional keep-out region.

Fig.~\ref{fig:query_placement} summarizes three design paradigms. (1) \emph{ACWM} uses a query-blind proposal and an action-conditioned physical predictor, with the query entering when predicted outcomes are scored. (2) The WAM-style \emph{Joint} model~\citep{ye2026world} conditions both candidate generation and outcome prediction on the query. (3) \emph{Modular} combines a query-aware proposal with the same action-conditioned predictor used by ACWM. All three ultimately evaluate predicted candidates under the same objective. Full training and evaluation details are given in App.~\ref{app:oc-setting}--\ref{app:oc-regret}.

\textbf{Joint specializes effectively to seen objectives, but its advantage largely disappears on held-out compositions.}
Across the five training objective families, Joint reduces regret by $0.027$ relative to ACWM. On held-out compositions, this difference shrinks to $0.004$ (Fig.~\ref{fig:query_placement_results}(a)). The same reduction in relative advantage remains under unnormalized costs (App.~\ref{app:oc-normalization}). Because ACWM and Joint also differ in predictive action exposure and in evaluating arbitrary supplied actions, we treat this as a comparison of complete interfaces rather than attributing the difference solely to query conditioning.

\phantomsection
\label{sec:modular_query}

\textbf{Separating where to look from what will happen.}
Two experiments suggest an asymmetric role for query information. Query access substantially improves candidate generation, while conditioning physical prediction on the query provides a much smaller relative advantage when objective terms are recombined. Modular separates these roles by using the query to focus candidate proposals while keeping physical prediction action-conditioned and reusable across objectives (Fig.~\ref{fig:query_placement}).

\textbf{Query-guided proposals improve candidate quality without coupling physical prediction to specific objectives.}
Modular reduces held-out regret by $0.011$ relative to ACWM and seen-objective regret by $0.014$ relative to Joint (Fig.~\ref{fig:query_placement_results}(b)). The held-out improvement comes primarily from candidate generation. Relative to ACWM, Modular reduces candidate-set regret by $0.024$ despite a $0.013$ increase in selection regret, yielding a net improvement of $0.011$ (Fig.~\ref{fig:query_placement_results}(a)). Query-aware proposals therefore uncover better actions even though selecting among the resulting candidates can become harder. This connects directly to Sec.~\ref{sec:planning}, where candidate discovery and final selection require different information.
Fig.~\ref{fig:query_placement_results}(c) provides a qualitative view of these interface differences, showing how their trajectories change across objective compositions, including one held out from training.

\begingroup
\setlength{\intextsep}{6pt}
\begin{figure}[h]
\centering
\vspace{-2mm}

{\fontsize{8.3}{10}\selectfont
\resizebox{0.95\linewidth}{!}{%
\begin{tabular}{@{}c@{\hspace{0.012626\linewidth}}c@{\hspace{0.015151\linewidth}}c@{}}
\makebox[0.287878\linewidth][c]{%
    \includegraphics[height=126bp]{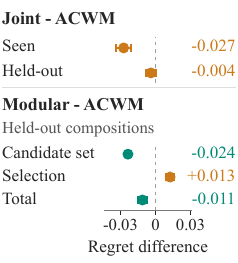}
} &
\makebox[0.241161\linewidth][c]{%
    \includegraphics[height=126bp]{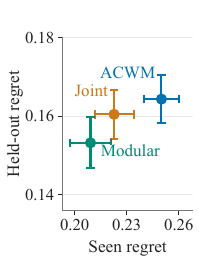}
} &
\makebox[0.443181\linewidth][c]{%
    \includegraphics[height=126bp]{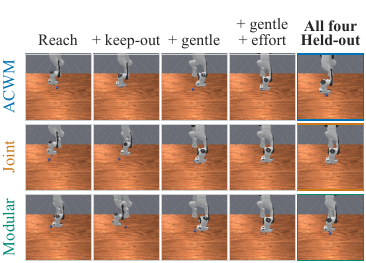}
} \\[1pt]
(a) Regret contrasts &
(b) Interface comparison &
(c) Objective compositions
\end{tabular}%
}\par}

\setlength{\abovecaptionskip}{4pt}
\setlength{\belowcaptionskip}{0pt}
\small

\caption{\textbf{Query placement shapes performance across objectives.}
\textbf{(a)} Regret contrasts and decomposition on seen and held-out compositions.
\textbf{(b)} Seen and held-out regret across the three interfaces.
\textbf{(c)} Qualitative executions across objective compositions, including one held out from training.}

\label{fig:query_placement_results}
\end{figure}
\endgroup

The Modular design retains a reusable physical-prediction interface. For a fixed observation and candidate action, its predicted outcome is query-independent and can therefore be re-scored across objectives. When the objective changes, the query-aware proposal can redirect search toward a new candidate set while reusing the same action-conditioned predictor (Fig.~\ref{fig:query_placement}). Empirically, this separation improves held-out regret over ACWM and seen-objective regret over Joint, while preserving query-independent physical predictions that can be reused across objectives.

\vspace{-2mm}
\section{Conclusion}
\label{sec:conclusion}
\vspace{-2mm}

What a world model must preserve for planning is determined not by predictive fidelity alone, but by the distinctions required for search and action choice. Our mechanism--response--decision hierarchy makes this dependence explicit: the query determines which physical variations matter and how precisely they must be resolved, while adaptive search may require information that final selection no longer needs. These requirements also shape model design. When capacity is insufficient to retain all predictive detail, planner-oriented supervision helps preserve the distinctions most useful for planning. When objectives vary, query-guided proposals can adapt search without making physical prediction objective-specific. Our Modular interface realizes this separation, improving planning while retaining reusable action-conditioned predictions across objectives. More broadly, an effective world model need not preserve every predictable detail uniformly. It should preserve the right distinctions at the resolution and stage where the planner needs them.

\section*{Acknowledgements}
Rongzhe Wei, Hans Hao-Hsun Hsu, and Pan Li are supported in part by the National Science Foundation (NSF) under awards PHY-2117997, IIS-2239565, IIS-2428777, and CCF-2402816; the NAIRR Pilot projects 250459 and 250487; the Google Cloud Research Credit (2026); the NVIDIA Academic Award (2026); and the Amazon Academic Award (2026).

\bibliography{iclr2027_conference}

@article{watter2015embed,
  title={Embed to control: A locally linear latent dynamics model for control from raw images},
  author={Watter, Manuel and Springenberg, Jost and Boedecker, Joschka and Riedmiller, Martin},
  journal={Advances in neural information processing systems},
  volume={28},
  year={2015}
}

@inproceedings{hafner2019learning,
  title={Learning latent dynamics for planning from pixels},
  author={Hafner, Danijar and Lillicrap, Timothy and Fischer, Ian and Villegas, Ruben and Ha, David and Lee, Honglak and Davidson, James},
  booktitle={International conference on machine learning},
  pages={2555--2565},
  year={2019},
  organization={PMLR}
}

@article{hafner2019dream,
  title={Dream to control: Learning behaviors by latent imagination},
  author={Hafner, Danijar and Lillicrap, Timothy and Ba, Jimmy and Norouzi, Mohammad},
  journal={arXiv preprint arXiv:1912.01603},
  year={2019}
}

@article{schrittwieser2020mastering,
  title={Mastering atari, go, chess and shogi by planning with a learned model},
  author={Schrittwieser, Julian and Antonoglou, Ioannis and Hubert, Thomas and Simonyan, Karen and Sifre, Laurent and Schmitt, Simon and Guez, Arthur and Lockhart, Edward and Hassabis, Demis and Graepel, Thore and others},
  journal={Nature},
  volume={588},
  number={7839},
  pages={604--609},
  year={2020},
  publisher={Nature Publishing Group UK London}
}

@article{ha2018world,
  title={World models},
  author={Ha, David and Schmidhuber, J{\"u}rgen},
  journal={arXiv preprint arXiv:1803.10122},
  volume={2},
  number={3},
  pages={440},
  year={2018}
}

@article{hansen2022temporal,
  title={Temporal difference learning for model predictive control},
  author={Hansen, Nicklas and Wang, Xiaolong and Su, Hao},
  journal={arXiv preprint arXiv:2203.04955},
  year={2022}
}

@inproceedings{hansen2024td,
  title={Td-mpc2: Scalable, robust world models for continuous control},
  author={Hansen, Nick and Su, Hao and Wang, Xiaolong},
  booktitle={International Conference on Learning Representations},
  volume={2024},
  pages={47376--47405},
  year={2024}
}

@article{hafner2025mastering,
  title={Mastering diverse control tasks through world models},
  author={Hafner, Danijar and Pasukonis, Jurgis and Ba, Jimmy and Lillicrap, Timothy},
  journal={Nature},
  volume={640},
  number={8059},
  pages={647--653},
  year={2025},
  publisher={Nature Publishing Group UK London}
}

@article{grimm2020value,
  title={The value equivalence principle for model-based reinforcement learning},
  author={Grimm, Christopher and Barreto, Andr{\'e} and Singh, Satinder and Silver, David},
  journal={Advances in neural information processing systems},
  volume={33},
  pages={5541--5552},
  year={2020}
}

@article{zhou2024dino,
  title={Dino-wm: World models on pre-trained visual features enable zero-shot planning},
  author={Zhou, Gaoyue and Pan, Hengkai and LeCun, Yann and Pinto, Lerrel},
  journal={arXiv preprint arXiv:2411.04983},
  year={2024}
}

@article{assran2025v,
  title={V-jepa 2: Self-supervised video models enable understanding, prediction and planning},
  author={Assran, Mido and Bardes, Adrien and Fan, David and Garrido, Quentin and Howes, Russell and Muckley, Matthew and Rizvi, Ammar and Roberts, Claire and Sinha, Koustuv and Zholus, Artem and others},
  journal={arXiv preprint arXiv:2506.09985},
  year={2025}
}

@article{alemi2016deep,
  title={Deep variational information bottleneck},
  author={Alemi, Alexander A and Fischer, Ian and Dillon, Joshua V and Murphy, Kevin},
  journal={arXiv preprint arXiv:1612.00410},
  year={2016}
}

@article{islam2022representation,
  title={Representation learning in deep rl via discrete information bottleneck},
  author={Islam, Riashat and Zang, Hongyu and Tomar, Manan and Didolkar, Aniket and Islam, Md Mofijul and Arnob, Samin Yeasar and Iqbal, Tariq and Li, Xin and Goyal, Anirudh and Heess, Nicolas and others},
  journal={arXiv preprint arXiv:2212.13835},
  year={2022}
}

@article{zhang2026world,
  title={World Action Planner: Generalizable Decision-Making with Action-Conditioned World Models},
  author={Zhang, Xiangcheng and Du, Yilun},
  journal={arXiv preprint arXiv:2607.27599},
  year={2026}
}

@article{lambert2020objective,
  title={Objective mismatch in model-based reinforcement learning},
  author={Lambert, Nathan and Amos, Brandon and Yadan, Omry and Calandra, Roberto},
  journal={arXiv preprint arXiv:2002.04523},
  year={2020}
}

@article{mandi2024decision,
  title={Decision-focused learning: Foundations, state of the art, benchmark and future opportunities},
  author={Mandi, Jayanta and Kotary, James and Berden, Senne and Mulamba, Maxime and Bucarey, Victor and Guns, Tias and Fioretto, Ferdinando},
  journal={Journal of Artificial Intelligence Research},
  volume={80},
  pages={1623--1701},
  year={2024}
}

@article{ye2026world,
  title={World action models are zero-shot policies},
  author={Ye, Seonghyeon and Ge, Yunhao and Zheng, Kaiyuan and Gao, Shenyuan and Yu, Sihyun and Kurian, George and Indupuru, Suneel and Tan, You Liang and Zhu, Chuning and Xiang, Jiannan and others},
  journal={arXiv preprint arXiv:2602.15922},
  year={2026}
}

@article{team2026motubrain,
  title={Motubrain: An advanced world action model for robot control},
  author={Team, MotuBrain and Xiang, Chendong and Bao, Fan and Liu, Haitian and Tan, Hengkai and Bi, Hongzhe and Li, James and Liu, Jiabao and Pang, Jingrui and Jing, Kiro and others},
  journal={arXiv preprint arXiv:2604.27792},
  year={2026}
}

@article{li2006towards,
  title={Towards a unified theory of state abstraction for MDPs.},
  author={Li, Lihong and Walsh, Thomas J and Littman, Michael L},
  journal={AI\&M},
  volume={1},
  number={2},
  pages={3},
  year={2006}
}

@article{ferns2012metrics,
  title={Metrics for finite Markov decision processes},
  author={Ferns, Norman and Panangaden, Prakash and Precup, Doina},
  journal={arXiv preprint arXiv:1207.4114},
  year={2012}
}

@inproceedings{farahmand2017value,
  title={Value-aware loss function for model-based reinforcement learning},
  author={Farahmand, Amir-massoud and Barreto, Andre and Nikovski, Daniel},
  booktitle={Artificial Intelligence and Statistics},
  pages={1486--1494},
  year={2017},
  organization={PMLR}
}

@inproceedings{silver2017predictron,
  title={The predictron: End-to-end learning and planning},
  author={Silver, David and Hasselt, Hado and Hessel, Matteo and Schaul, Tom and Guez, Arthur and Harley, Tim and Dulac-Arnold, Gabriel and Reichert, David and Rabinowitz, Neil and Barreto, Andre and others},
  booktitle={International conference on machine learning},
  pages={3191--3199},
  year={2017},
  organization={PMLR}
}

@article{oh2017value,
  title={Value prediction network},
  author={Oh, Junhyuk and Singh, Satinder and Lee, Honglak},
  journal={Advances in neural information processing systems},
  volume={30},
  year={2017}
}

@article{cheng2026sage,
  title={SAGE: Subgoal-Conditioned Action Generation for Latent World Model Planning},
  author={Cheng, Letian and Zhang, Qi and Wang, Yisen},
  journal={arXiv preprint arXiv:2607.17973},
  year={2026}
}

@article{thorpe2026physically,
  title={Physically Viable World Models: A Case for Query-Conditioned Embodied AI},
  author={Thorpe, Adam J and Tretiakov, Stepan and Hsiao, Cheng-Hsi and Low, Su Ann and Li, Xingjian and Iqbal, Hassan and Bhatt, Neel P and Topcu, Ufuk and Kumar, Krishna},
  journal={arXiv preprint arXiv:2605.30542},
  year={2026}
}

@article{oquab2023dinov2,
  title={Dinov2: Learning robust visual features without supervision},
  author={Oquab, Maxime and Darcet, Timoth{\'e}e and Moutakanni, Th{\'e}o and Vo, Huy and Szafraniec, Marc and Khalidov, Vasil and Fernandez, Pierre and Haziza, Daniel and Massa, Francisco and El-Nouby, Alaaeldin and others},
  journal={arXiv preprint arXiv:2304.07193},
  year={2023}
}

@article{terver2025drives,
  title={What drives success in physical planning with joint-embedding predictive world models?},
  author={Terver, Basile and Yang, Tsung-Yen and Ponce, Jean and Bardes, Adrien and LeCun, Yann},
  journal={arXiv preprint arXiv:2512.24497},
  year={2025}
}

@article{tao2024maniskill3,
  title={Maniskill3: Gpu parallelized robotics simulation and rendering for generalizable embodied ai},
  author={Tao, Stone and Xiang, Fanbo and Shukla, Arth and Qin, Yuzhe and Hinrichsen, Xander and Yuan, Xiaodi and Bao, Chen and Lin, Xinsong and Liu, Yulin and Chan, Tse-kai and others},
  journal={arXiv preprint arXiv:2410.00425},
  year={2024}
}

@article{arumugam2022deciding,
  title={Deciding what to model: Value-equivalent sampling for reinforcement learning},
  author={Arumugam, Dilip and Van Roy, Benjamin},
  journal={Advances in neural information processing systems},
  volume={35},
  pages={9024--9044},
  year={2022}
}

@inproceedings{nair2020goal,
  title={Goal-aware prediction: Learning to model what matters},
  author={Nair, Suraj and Savarese, Silvio and Finn, Chelsea},
  booktitle={International conference on machine learning},
  pages={7207--7219},
  year={2020},
  organization={PMLR}
}

@article{zhang2020learning,
  title={Learning invariant representations for reinforcement learning without reconstruction},
  author={Zhang, Amy and McAllister, Rowan and Calandra, Roberto and Gal, Yarin and Levine, Sergey},
  journal={arXiv preprint arXiv:2006.10742},
  year={2020}
}

@article{chi2025diffusion,
  title={Diffusion policy: Visuomotor policy learning via action diffusion},
  author={Chi, Cheng and Xu, Zhenjia and Feng, Siyuan and Cousineau, Eric and Du, Yilun and Burchfiel, Benjamin and Tedrake, Russ and Song, Shuran},
  journal={The International Journal of Robotics Research},
  volume={44},
  number={10-11},
  pages={1684--1704},
  year={2025},
  publisher={Sage Publications Sage UK: London, England}
}

@article{rubinstein1999cross,
  title={The cross-entropy method for combinatorial and continuous optimization},
  author={Rubinstein, Reuven},
  journal={Methodology and computing in applied probability},
  volume={1},
  number={2},
  pages={127--190},
  year={1999},
  publisher={Springer}
}

@article{williams2017model,
  title={Model predictive path integral control: From theory to parallel computation},
  author={Williams, Grady and Aldrich, Andrew and Theodorou, Evangelos A},
  journal={Journal of Guidance, Control, and Dynamics},
  volume={40},
  number={2},
  pages={344--357},
  year={2017},
  publisher={American Institute of Aeronautics and Astronautics}
}

@article{m2023model,
  title={Model-based reinforcement learning: A survey},
  author={M. Moerland, Thomas and Broekens, Joost and Plaat, Aske and M. Jonker, Catholijn},
  journal={Foundations and Trends in Machine Learning},
  volume={16},
  number={1},
  pages={1--118},
  year={2023},
  publisher={Emerald Publishing Limited}
}
\bibliographystyle{iclr2027_conference}
\newpage
\onecolumn
\appendix

\renewcommand{\topfraction}{0.92}
\renewcommand{\bottomfraction}{0.85}
\renewcommand{\textfraction}{0.06}
\renewcommand{\floatpagefraction}{0.78}
\setcounter{topnumber}{4}
\setcounter{bottomnumber}{3}
\setcounter{totalnumber}{7}
\setlength{\textfloatsep}{7pt plus 1pt minus 1pt}
\setlength{\floatsep}{5pt plus 1pt minus 1pt}
\setlength{\intextsep}{6pt plus 1pt minus 1pt}
\setlength{\abovecaptionskip}{2pt}
\setlength{\belowcaptionskip}{0pt}

\begingroup
\setlength{\parindent}{0pt}
\setlength{\parskip}{0pt}
\hypersetup{hidelinks}
{\color{black!65}\hrule height .65pt}
\vspace{9pt}
{\centering
 {\fontsize{9.5}{11}\selectfont\scshape\color{black!60} Supplementary\enspace Material\par}
 \vspace{4pt}
 {\fontsize{28}{31}\selectfont\bfseries Appendix\par}
}
\vspace{10pt}
{\color{black!65}\hrule height .35pt}
\vspace{10pt}
{\centering\large\bfseries Contents\par}
\vspace{6pt}
\etocdepthtag.toc{mtappendix}
\etocsettagdepth{mtchapter}{none}
\etocsettagdepth{mtappendix}{subsection}
\etocsettocstyle{}{}
\small
\setlength{\parskip}{0pt}
\makeatletter
\renewcommand*\l@section[2]{\vspace{3pt}\begingroup\renewcommand*\@dotsep{10000}\@dottedtocline{1}{0em}{1.6em}{\bfseries #1}{\bfseries #2}\endgroup}
\renewcommand*\l@subsection[2]{\@dottedtocline{2}{1.6em}{2.7em}{#1}{#2}}
\makeatother
\setcounter{tocdepth}{2}
\tableofcontents
\endgroup
\newpage

\vspace{-2mm}
\section{Formal Properties and Planner Requirements}
\vspace{-2mm}

\label{app:response-regret}

This section supplies the assumptions and proofs for the hierarchy in Sec.~\ref{sec:hierarchy} and the response-to-decision guarantee in Sec.~\ref{sec:query_structure}. Throughout, lower cost is better. The query, candidate set, and initial conditions are fixed when comparing physical worlds.

\vspace{-1mm}
\subsection{Exact Refinement and the Role of Ties}
\vspace{-1mm}

Fix a deterministic tie rule shared across worlds. For finite candidate sets, any fixed ordering of candidates suffices. Identical physical dynamics yield the same outcome for every candidate and hence the same candidate-cost profile. Equal profiles yield equal minimizer sets, and the shared tie rule selects the same action. This demonstrates $\Pi_M\preceq\Pi_R^{C,q}\preceq\Pi_D^{C,q}$. Mechanism sufficiency here concerns equality of dynamics, not necessarily equality of parameter vectors when a parameterization is redundant. Approximate parameter recovery requires an additional continuity bound before it implies response accuracy.

Sharing an optimal action is not itself an equivalence relation. The minimizer sets $\{A_1\}$, $\{A_1,A_2\}$, and $\{A_2\}$ overlap for adjacent pairs but not for the first and third. The exact decision partition therefore uses a selected minimizer, whereas an approximate decision cover only requires a common low-regret action in each covering set.

\vspace{-1mm}
\subsection{Response Accuracy Certifies Low Regret}
\vspace{-1mm}

\begin{proof}[Proof of Lem.~\ref{lem:response_decision}]
    Let $A^\star$ minimize the true cost and let $\widehat A$ minimize the predicted cost. Suppressing the fixed query, world, and candidate set gives
    \begin{align}
    J(\widehat A)-J(A^\star)
    &=J(\widehat A)-\widehat J(\widehat A)
    +\widehat J(\widehat A)-\widehat J(A^\star)
    +\widehat J(A^\star)-J(A^\star)\\
    &\leq \epsilon+0+\epsilon=2\epsilon.
    \end{align}
    For each radius-$\epsilon$ response-cover element, select an action minimizing its common decoded profile. Every world covered by that element admits this action with regret at most $2\epsilon$. These sets form a decision cover of no larger size, proving $N_D^{C,q}(2\epsilon)\leq N_R^{C,q}(\epsilon)$. The argument assumes that the decoded profile attains its minimum. This is automatic for finite $C$, and holds for continuous profiles on a compact candidate interval as used in App.~\ref{app:collision-covers}.
\end{proof}
The converse fails even with two candidates. Let $J_b(A_1)=b$ and $J_b(A_2)=b+1$. Selecting $A_1$ requires no information about $b$ and has zero regret for every world. Approximating both costs uniformly still requires resolving the offset. On a bounded interval for $b$, the response cover can be arbitrarily large while the exact decision cover has size one.

\textbf{Structural complexity and measured rate.} An $N$-element cover admits a fixed-length description using at most $\lceil \log_2 N\rceil$ bits, while the entropy of a cell label also depends on the world distribution. Rate--distortion instead optimizes a statistical information measure, whereas our learned studies report the coded stream produced by a particular encoder and lossless coding scheme. These quantities capture different notions of complexity. In particular, a negative learned response-minus-decision rate saving does not contradict the cover inequality.

\vspace{-1mm}
\subsection{What an Adaptive Planner Must Preserve}
\vspace{-1mm}

Sec.~\ref{sec:search_selection} shows that information sufficient for final selection can be insufficient during adaptive search. The reason is that different planner operations depend on different summaries of the candidate costs. With the candidate actions and planner state fixed, final selection needs only the minimizing candidate, whereas adaptive planning methods such as CEM~\citep{rubinstein1999cross} and MPPI~\citep{williams2017model} may require richer summaries of the candidate costs: CEM uses the elite subset, while MPPI uses normalized relative cost weights (Tab.~\ref{tab:planner-information}).

\begin{table}[htbp]
\centering
\small
\caption{Information sufficient for representative planner operations, with candidate actions, planner state, and random draws fixed.}
\label{tab:planner-information}
\vspace{1mm}
\begin{tabular}{P{.20\linewidth}P{.27\linewidth}P{.43\linewidth}}
\toprule
\textbf{Planner operation} & \textbf{Required cost information} & \textbf{Why it is sufficient} \\
\midrule
Final selection
& Minimizing candidate index
& Determines the action selected from the current candidate set. \\
\midrule
CEM update
& Elite subset (top-$k$ lowest-cost candidates)
& The actions in the elite set determine the refitted sampling distribution. \\
\midrule
MPPI update
& Normalized weights
  $w_i \propto \exp(-J_i/\lambda)$
& The update averages candidate actions using these weights. Adding the same constant to every candidate cost leaves all normalized weights unchanged, so the absolute cost offset is unnecessary.\\
\bottomrule
\end{tabular}
\end{table}

The complete candidate-cost profile is sufficient for all three operations, but generally contains more information than the planner actually uses. In adaptive search, however, the relevant update information must be preserved throughout the search: changing an intermediate update changes the subsequent sampling distribution and therefore the candidates considered later.

\vspace{-2mm}
\section{Collision: Analytic Structure and Learned Representations}
\vspace{-2mm}

\label{app:collision-covers}\label{app:collision-learning}

While Prop.~\ref{prop:collision} analytically establishes the query-dependent scaling of cover complexity, our empirical study measures the actual coded representations across disjoint physical worlds. Throughout, we strictly distinguish the theoretical cover complexity in continuous action spaces from the empirical rate achieved by a finite-candidate learner.

\vspace{-1mm}
\subsection{Query-Relevant Physical Quantities}
\vspace{-1mm}

Let $\theta=(m_1,m_2,e)$ denote the collision mechanism, with body 2 initially at rest. For an incoming velocity $v$, conservation of momentum and restitution give
\begin{equation}
m_1v=m_1v_1'+m_2v_2',\qquad v_2'-v_1'=ev,
\end{equation}
and hence
\begin{equation}
v_2'=\kappa v,
\qquad
\kappa=\frac{(1+e)m_1}{m_1+m_2}.
\end{equation}
If instead body 1 receives an impulse $u$, then $m_1v=u$ and
\begin{equation}
v_2'=\lambda u,
\qquad
\lambda=\frac{1+e}{m_1+m_2}.
\end{equation}
Thus, the fixed-velocity and fixed-impulse settings depend on different physical quantities, $\kappa$ and $\lambda$. Over $m_1,m_2\in[0.5,2]$ and $e\in[0,1]$, their ranges are $\kappa\in[0.2,1.6]$ and $\lambda\in[0.25,2]$.

Neither quantity determines the other over the mechanism family. Mechanisms $(1,1,0)$ and $(0.75,0.75,0)$ share $\kappa=1/2$ but have different $\lambda$, while $(1,1,0)$ and $(2,1,0.5)$ share $\lambda=1/2$ but have different $\kappa$. Therefore, queries $q_1$ and $q_2$ can require different physical distinctions within the same mechanism family.

\vspace{-1mm}
\subsection{Threshold and Tracking Covers}
\vspace{-1mm}

\textbf{Threshold decisions.}
Write the range of $\kappa$ as $[a,b]$, with $0<a<b$, and assume $c_1/v_0\in(a,b)$. For $q_1$, the candidate-cost profile is $(0,c_1-\kappa v_0)$, and the two candidates are uniquely optimal on opposite sides of the threshold. The two actions therefore suffice to cover all worlds at zero regret, while no single action does, giving $N_D^{C,q_1}(0)=2$. A radius-$\epsilon$ response-cover element can span at most $2\epsilon/v_0$ in $\kappa$, and intervals of this width suffice. Hence,
\begin{equation}
N_R^{C,q_1}(\epsilon)
=
\left\lceil
\frac{v_0(b-a)}{2\epsilon}
\right\rceil.
\end{equation}
The same argument applies to $q_2$ with $\lambda,u_0,c_2$. For sufficiently small positive regret tolerance, both actions remain necessary, whereas the response cover continues to grow as $\epsilon\to0$. The matched-tolerance log-cover gap therefore diverges.

\textbf{Continuous tracking.}
Let $t=v_{\mathrm{tar}}>0$ and let $C=[v_{\mathrm{lo}},v_{\mathrm{hi}}]$, with $v_{\mathrm{lo}}>0$, contain every optimum $t/\kappa$. Assume one endpoint lies on the same side of all optima, for example $v_{\mathrm{hi}}\geq t/a$. The optimal cost is zero. For $\Delta<t$, an action $v$ has regret at most $\Delta$ only when
\begin{equation}
\kappa
\in
\left[
\frac{t-\Delta}{v},
\frac{t+\Delta}{v}
\right]
\cap[a,b].
\end{equation}
This interval has length at most $2\Delta/v_{\mathrm{lo}}$, so $\Omega(1/\Delta)$ actions are necessary. Conversely, partition $[a,b]$ into intervals of width at most $2a\Delta/t$ and choose $v=t/\kappa_0$ at each midpoint $\kappa_0$. Every world in that interval then has regret at most
$(t/a)|\kappa-\kappa_0|\leq\Delta$, giving the matching upper bound.

For response profiles, the reverse triangle inequality gives
\begin{equation}
\sup_{v\in C}
\left|
|\kappa v-t|-|\kappa'v-t|
\right|
\leq
v_{\mathrm{hi}}|\kappa-\kappa'|.
\end{equation}
Under the endpoint condition above, equality is attained at the common upper endpoint. Thus
$N_R^{C,q_3}(\epsilon)=\Theta(1/\epsilon)$ and
$N_D^{C,q_3}(\Delta)=\Theta(1/\Delta)$, so their log-cover difference remains bounded at $\Delta=2\epsilon$. These asymptotics concern the continuous candidate interval and do not apply to the fixed bank of 64 actions used in the learned experiment, whose exact decision cover is finite.

\vspace{-1mm}
\subsection{Worlds, Observations, and Queries}
\vspace{-1mm}

We sample independent $\log m_1,\log m_2\sim\mathcal U[\log 0.5,\log 2]$ and $e\sim\mathcal U[0,1]$, yielding $200{,}000/20{,}000/50{,}000$ disjoint training/validation/test mechanisms. Each world is observed through two noiseless collisions on reset copies of the same mechanism: one with unit incoming velocity and one with unit impulse. The resulting observation is $((m_1-em_2)/(m_1+m_2),\kappa,\lambda)$. These three quantities are standardized using training statistics and zero-padded to width 256 before encoding. The model receives these observations and the query, but not the mechanism $\theta$.

Queries $q_1$ and $q_2$ use the two actions defined in Sec.~\ref{sec:collision}. With $v_0=u_0=v_{\mathrm{tar}}=1$, we set $c_1$ and $c_2$ to the median values of $\kappa$ and $\lambda$, respectively, in a separate 100,000-world calibration sample, giving $c_1=0.7259$ and $c_2=0.7041$. Query $q_3$ uses 64 geometrically spaced velocities on $[0.59375,5.25]$. Since $v_{\mathrm{tar}}=1$ and $\kappa\in[0.2,1.6]$, the continuous optimum $v_{\mathrm{tar}}/\kappa$ lies in $[0.625,5]$; we extend this range by 5\% at both ends to ensure that all continuous optima lie strictly inside the candidate interval. The candidate set is fixed across worlds and independent of the observed record. All normalization statistics are fixed before test evaluation.

We normalize regret by a query-specific scale $s_q$, defined as the median positive candidate regret on the training worlds plus $10^{-8}$. The resulting scales are approximately $0.1749$, $0.1689$, and $0.5320$ for $q_1$, $q_2$, and $q_3$, respectively. Reported regret is $\operatorname{Reg}_q(A;\theta,C)/s_q$, following the paper's cost-minimization convention.

\vspace{-1mm}
\subsection{Representation, Supervision, and Common Scorer}
\label{app:collision-model}
\vspace{-1mm}

The encoder maps each physical observation to a variable-length discrete message using product vector quantization (VQ). It selects an active prefix of at most 32 positions, each with a 256-entry codebook of 32-dimensional vectors. The forward pass uses discrete lengths and symbols, with straight-through gradients for the encoder and exponential-moving-average updates for the codebooks. Only the reconstructed symbols and encoded length are available downstream, continuous encoder features cannot bypass the discrete message. Tab.~\ref{tab:compression-architecture} summarizes the architecture, which is also used in the Double Pendulum study.

\begin{table}[!hbp]
\centering
\small
\caption{Shared coded-rate architecture.}
\label{tab:compression-architecture}
\vspace{1mm}
\begin{tabular}{P{.24\linewidth}P{.68\linewidth}}
\toprule
\textbf{Component} & \textbf{Architecture and information access} \\
\midrule
Physical encoder
& Physical input and public embeddings, followed by three width-256 SiLU/LayerNorm layers with separate symbol and length outputs. \\
\midrule
Public embeddings
& Query and known-condition embeddings, each with 128 dimensions. Fixed auxiliary features contain no world information. \\
\midrule
Message decoder
& Active codewords concatenated with zeros for inactive positions and a 16-dimensional length embedding, then projected to 128 dimensions. \\
\midrule
Candidate head
& Three-layer $[\mathrm{message},\mathrm{public},\mathrm{action}]\to256\to256\to1$ SiLU network, with a three-layer width-128 action embedding. \\
\midrule
Entropy model
& Two causal Transformer blocks with width 128, four heads, feed-forward width 512, and no dropout. The model predicts message length followed by active symbols conditioned on preceding symbols and public inputs. \\
\bottomrule
\end{tabular}
\end{table}

Let $\widehat J_q(A_j)$ denote the response head's predicted cost and $s_j$ the decision head's preference score. With $\pi_j=\operatorname{softmax}(s)_j$, the two native objectives are
\begin{align}
\mathcal L_{\mathrm{resp}} &= \mathbb E\!\left[\max_{j\leq |C|}\frac{|\widehat J_q(A_j)-J_q(A_j;\theta)|}{s_q}\right], \label{eq:collision-response-loss}\\
\mathcal L_{\mathrm{dec}} &= \mathbb E\!\left[\sum_j\pi_j\frac{\operatorname{Reg}_q(A_j;\theta,C)}{s_q}+\tau_q\sum_j\pi_j\log\pi_j\right]. \label{eq:collision-decision-loss}
\end{align}
Decision supervision minimizes expected normalized regret under the score-induced distribution. We include an entropy bonus, $\tau_q\sum_j\pi_j\log\pi_j$, to discourage premature collapse of this distribution during training, evaluation still selects the highest-scoring candidate. The softmax temperature is fixed to one, and $\tau_q$ is set to $0.3$, $0.1$, and $0$ for $q_1$, $q_2$, and $q_3$, respectively. Response and decision models use the same encoder and candidate-head architecture. Exact simulator costs provide supervision, while mechanism labels are never given to the encoder.

For each selected compression checkpoint, we freeze the encoder, codebooks, and symbol-to-embedding decoder, then train a fresh common scorer from scratch on the same training worlds using Eq.~\ref{eq:collision-decision-loss}. The scorer receives only the decoded message, query, and candidate action, with the same query-specific $\tau_q$ for both representations. Its validation loss selects the scorer checkpoint, and evaluation selects the candidate with the largest score using a fixed deterministic tie rule. This matched readout isolates differences in the information retained by the representation rather than differences between the original task-specific heads.

\vspace{-1mm}
\subsection{Compression and Lossless Coding}
\label{app:collision-compression}
\vspace{-1mm}

Compression starts from the best task-only checkpoint selected on validation data. We use two reference losses to define how much task performance may degrade as the representation is compressed. $D_{\mathrm{full}}$ is the pre-compression task loss obtained from the learned representation without a rate constraint. $D_0$ is the loss of a zero-information baseline that does not distinguish between physical worlds. Under response supervision, this baseline predicts the same constant candidate-cost vector for every world. Under decision supervision, it always selects the same fixed action, chosen to minimize average training cost. These reference losses are evaluated separately on the training and validation splits.

For nine tolerance levels spanning $\rho=0.01$ to $0.95$, we define
\begin{equation}
D_{\max}(\rho)=D_{\mathrm{full}}+\rho(D_0-D_{\mathrm{full}}).
\end{equation}
Smaller $\rho$ keeps the compressed representation close to its pre-compression task performance, while larger $\rho$ permits more degradation toward the zero-information baseline.

At each tolerance, we reduce the representation rate while penalizing task loss once it approaches the allowed degradation:
\begin{align}
\mathcal L_{\mathrm{compress}} &= \widehat R+1000\left[\frac{D_{\mathrm{batch}}-D_{\mathrm{full},\mathrm{train}}}{\eta}-\frac12\right]_+ +0.25\mathcal L_{\mathrm{VQ}},\\
\eta &= \rho(D_{0,\mathrm{train}}-D_{\mathrm{full},\mathrm{train}}),
\end{align}
where $[x]_+=\max(x,0)$, $\widehat R$ is the differentiable rate surrogate, and $D_{\mathrm{batch}}$ is the native minibatch loss from Eq.~\ref{eq:collision-response-loss} or~\ref{eq:collision-decision-loss}. The hinge penalty begins halfway between the pre-compression loss and the allowed tolerance, providing a margin before $D_{\max}$ is reached.

At each compression stage, we refit the entropy model and select the eligible checkpoint with the lowest validation symbol negative log likelihood (NLL), subject to $D_{\mathrm{val}}\leq D_{\max,\mathrm{val}}$. The checkpoint from the preceding stage remains eligible, and the selected representation initializes the next, looser stage. Entropy refitting updates only the coding probabilities while keeping the representation fixed. Both supervision settings use the training budgets in Tab.~\ref{tab:collision-training}, with at least 2,000 steps per compression stage.

Unless noted otherwise, optimization uses AdamW with learning rate $3\times10^{-4}$, weight decay $10^{-4}$, batch size 1024, and gradient clipping at one. Entropy refitting uses Adam at $10^{-3}$.

\begin{table}[h]
\centering
\small
\caption{Training budgets and checkpoint selection for the Collision comparison.}
\label{tab:collision-training}
\vspace{1mm}
\setlength{\tabcolsep}{4pt}
\begin{tabular}{lrrP{.40\linewidth}}
\toprule
\textbf{Stage} & \textbf{Maximum steps} & \textbf{Check interval} & \textbf{Checkpoint rule} \\
\midrule
Task-only training
& 8,000 & 500
& Lowest native validation loss. \\
\midrule
Each compression stage
& 4,000 & 500
& Lowest refitted validation symbol NLL among checkpoints satisfying the native-loss tolerance. \\
\midrule
Entropy refit
& 3,000 & 250
& Lowest validation symbol NLL, with patience three. \\
\midrule
Common scorer
& 8,000 & 500
& Lowest validation decision loss, with patience five after 3,000 steps. \\
\bottomrule
\end{tabular}
\end{table}

Learning rates use linear warmup over five training-set passes followed by cosine decay. The VQ commitment penalty is ramped during task-only training and fixed during compression.

Reported rates are the lengths of the complete lossless coded streams, including message-length symbols and coding overhead, divided by the number of contexts. Each stream is decoded and verified against the original symbols. Shared model parameters and public side information are excluded from the rate. Validation NLL is used for checkpoint selection, while reported rates are computed from the actual range-coded streams. The zero-rate endpoint sends no message and uses a single fixed action chosen to minimize average training cost. Its test regret anchors the rate--regret frontier.

\vspace{-1mm}
\subsection{Frontier Construction and Supported Comparisons}
\label{app:collision-frontiers}
\vspace{-1mm}

For each training seed, we retain the nondominated rate--regret points and linearly interpolate between adjacent points, including the fixed-action zero-rate endpoint. The required rate at regret tolerance $\delta$ is defined only when $\delta$ lies within the achieved frontier. For a seed supported by both supervision methods, the paired rate saving is the response-supervised required rate minus the decision-supervised required rate at the same tolerance. Five training seeds are run per method, and plots average the supported seed-level quantities at each tolerance. Standard deviations summarize variation across training seeds.

Tab.~\ref{tab:collision-primary} reports the saved grid points nearest the nominal regret tolerances. Both threshold queries show positive mean paired savings at the reported benchmark tolerances, with all five training seeds contributing. The ordering can reverse elsewhere along the frontier. For $q_3$ at $\delta\approx0.001$, some trained frontiers do not reach the target regret, so their required rates are undefined. Marginal response and decision rates can therefore average over different supported seed sets from the paired saving, and the reported values need not subtract exactly.

\begin{table}[h]
\centering
\small
\caption{Required coded rates and paired response-minus-decision savings (bits per context). Each method uses five training seeds; $\dagger$ marks partial frontier support.}
\label{tab:collision-primary}
\vspace{1mm}
\setlength{\tabcolsep}{4pt}
\begin{tabular}{@{}crrr@{\hspace{12pt}}rrr@{\hspace{12pt}}rrr@{}}
\toprule
& \multicolumn{3}{c}{$q_1$: velocity threshold}
& \multicolumn{3}{c}{$q_2$: impulse threshold}
& \multicolumn{3}{c}{$q_3$: tracking} \\
\cmidrule(lr){2-4}\cmidrule(lr){5-7}\cmidrule(l){8-10}
$\delta\approx$
& $R_{\mathrm{resp}}$ & $R_{\mathrm{dec}}$ & $S_q$
& $R_{\mathrm{resp}}$ & $R_{\mathrm{dec}}$ & $S_q$
& $R_{\mathrm{resp}}$ & $R_{\mathrm{dec}}$ & $S_q$ \\
\midrule
$0.1$
& $1.30$ & $0.92$ & $+0.38$
& $1.21$ & $0.93$ & $+0.27$
& $2.59$ & $2.66$ & $-0.08$ \\
$0.01$
& $1.98$ & $1.36$ & $+0.62$
& $1.66$ & $1.09$ & $+0.58$
& $5.53$ & $6.55$ & $-1.02$ \\
$0.001$
& $3.32$ & $1.76$ & $+1.56$
& $2.74$ & $1.13$ & $+1.61$
& $9.58^{\dagger}$ & $9.97^{\dagger}$ & $-0.32^{\dagger}$ \\
\bottomrule
\end{tabular}
\par\vspace{2pt}
\begin{minipage}{\linewidth}\footnotesize
$\dagger$ \emph{Support (response/decision/paired):} $2/3/2$. The remaining trained frontiers do not attain this tolerance.
\end{minipage}
\end{table}

Some configurations near the tight-regret end of the frontier reach the fixed message-length cap. Rates in this regime can therefore also reflect finite representation capacity and optimization. More generally, the reported rates characterize the learned encoder and coding procedure used in the experiment.

\vspace{-2mm}
\section{Double Pendulum: Learned Dynamics and Decision Requirements}
\label{app:dp}
\vspace{-2mm}

This study examines how the information required for planning changes with decision resolution and planning horizon while keeping the learned physical representation fixed.

\vspace{-1mm}
\subsection{Physical System, Observations, and Planning Query}
\label{app:dp-wm}
\vspace{-1mm}

We consider a planar double pendulum with point masses and massless rigid rods. The hidden mechanism is
$\theta=(r_m,l_1,r_l)$, where $r_m=m_2/m_1$ is the mass ratio and $r_l=l_2/l_1$ is the link-length ratio. We fix $m_1=1$, set $m_2=r_m m_1$ and $l_2=r_l l_1$, and use $g=9.81$. Across mechanisms, $r_m$ is sampled log-uniformly from $[0.5,2]$, $l_1$ uniformly from $[0.8,1.2]$, and $r_l$ log-uniformly from $[0.7,1.3]$.

Angles $\vartheta_1,\vartheta_2$ are measured from the downward vertical, with angular velocities $\omega_i=\dot\vartheta_i$. Letting $d=\vartheta_1-\vartheta_2$, the kinetic and potential energies are
\begin{align}
E_{\mathrm{kin}} &= \tfrac12(m_1+m_2)l_1^2\omega_1^2+\tfrac12m_2l_2^2\omega_2^2+m_2l_1l_2\cos(d)\omega_1\omega_2,\\
E_{\mathrm{pot}} &= (m_1+m_2)gl_1(1-\cos\vartheta_1)+m_2gl_2(1-\cos\vartheta_2).
\end{align}
The Euler--Lagrange equations give
\begin{equation}
\begin{bmatrix}
(m_1+m_2)l_1 & m_2l_2\cos d\\
l_1\cos d & l_2
\end{bmatrix}
\begin{bmatrix}
\dot\omega_1\\
\dot\omega_2
\end{bmatrix}
=
\begin{bmatrix}
-m_2l_2\sin(d)\omega_2^2-(m_1+m_2)g\sin\vartheta_1\\
l_1\sin(d)\omega_1^2-g\sin\vartheta_2
\end{bmatrix}.
\end{equation}
The matrix determinant is $l_1l_2(m_1+m_2\sin^2d)>0$. We measure the height of the second mass as
\begin{equation}
h_2=l_1(1-\cos\vartheta_1)+l_2(1-\cos\vartheta_2).
\end{equation}

Each mechanism is observed through two query-independent calibration trajectories, initialized at $(0.60,-0.35,0,0.50)$ and $(-0.45,0.70,0.70,-0.55)$. Each trajectory contains 32 equally spaced observations over $1.5$\,s. Concatenating the four state variables across both trajectories gives a 256-dimensional observation history $h_t$. The direct-history control uses these same observations, standardized with training statistics.

Planning starts from a new initial state. We independently sample $\vartheta_{1,0},\vartheta_{2,0}\sim\mathcal U[-1.2,1.2]$ and $\omega_{1,0}\sim\mathcal U[-1,1]$, while the planner chooses the second-joint initial velocity $u=\omega_{2,0}$. We use $10{,}000$ training, $2{,}000$ validation, and $4{,}000$ test mechanisms, with disjoint mechanisms across splits.

\textbf{Planning query and candidates.}
The query $q=(h^\star,T)$ specifies a target height $h^\star$ and a planning horizon $T\in\{1,2,4\}$\,s. For a candidate initial velocity $u=\omega_{2,0}$, the cost is
\begin{equation}
J_q(u;\theta)=\left[\frac{\max_{0\leq t\leq T}h_2(t;u,\theta)-h^\star}{s_h}\right]^2.
\end{equation}
We fix the target range and height scale using a separate design pool of 4,096 mechanisms, disjoint from the training, validation, and test splits. Pooling the maximum heights reached over $4$\,s by all 64 candidate actions, we set the target interval to the 20th--80th percentiles of these responses,
$h^\star\in[0.3617,1.0740]$, and sample targets uniformly from this interval. We set $s_h=0.4244$ using $1.4826$ times the median absolute deviation of the same response distribution, providing a robust scale for the height error. The same $s_h$ is used for all planning horizons.

The full candidate set contains 64 equally spaced velocities on $[-3,3]$ rad/s. To vary decision resolution while keeping the action range fixed, we use nested subsets with $K\in\{2,8,64\}$. The eight-candidate set uses zero-based indices $0,9,18,27,36,45,54,63$, while the two-candidate set uses indices $18$ and $45$, corresponding to approximately $\pm1.286$ rad/s. These candidate sets are fixed across worlds.

To isolate the effect of decision resolution, response supervision always predicts costs for all 64 actions, whereas decision supervision and evaluation operate on the selected $K$-candidate subset. Thus changing $K$ alters the decision being made without changing the dimensionality of the response target.

We use one common regret scale across all resolution and horizon comparisons: approximately $6.760$, defined as the 95th percentile of the candidate-cost range over the 64-action bank on training contexts, pooled across $T\in\{1,2,4\}$\,s. This places all Double Pendulum regret tolerances $\delta$ on the same scale.

\vspace{-1mm}
\subsection{Learning and Validating the Physical Representation}
\label{app:dp-representation}
\vspace{-1mm}

The physical encoder represents each mechanism from its two calibration trajectories. The 32 observations from each trajectory form 64 tokens, each containing standardized state coordinates, observation time, and trajectory identity. Four pre-LN Transformer blocks (width 128, four heads, FFN width 512, GELU) aggregate these tokens through a learned summary token into a 256-dimensional predictive state $z_t$. The encoder receives only the calibration trajectories and does not observe the planning query, candidate resolution, target height, or mechanism parameters.

A causal Transformer with the same block architecture conditions on $z_t$ and the preceding states to predict the next standardized state increment. The encoder and dynamics model contain approximately 0.829M and 0.853M parameters, respectively. To train the dynamics predictor, we generate auxiliary rollouts from additional initial conditions under the same mechanism splits: eight per training mechanism and four per validation/test mechanism, giving $80{,}000/8{,}000/16{,}000$ rollouts. Each rollout contains 200 transitions at 0.02\,s intervals. These rollouts provide dynamics supervision, the downstream compression study still treats each mechanism as a single planning context.

Training minimizes mean squared error on standardized state increments under teacher forcing. We use AdamW with learning rate $3\times10^{-4}$, weight decay $10^{-4}$, batch size 256, gradient clipping at one, a 1,000-step warmup, and cosine decay. Model selection uses autoregressive rollout error: every 1,000 steps, we evaluate RMSE at horizons of $1$, $2$, and $4$\,s on 512 fixed validation rollouts and select the checkpoint with the lowest average RMSE. After 10,000 training steps, early stopping uses patience five, with a maximum of 30,000 steps.

\begin{table}[h]
\centering
\small
\caption{Temporal-model training and autoregressive-rollout fidelity. RMSE is computed on standardized states, height MAE at $T=4$\,s, and planning columns report winner accuracy / normalized regret.}
\label{tab:dp-wm-fidelity}
\vspace{1mm}
\resizebox{\linewidth}{!}{%
\begin{tabular}{@{}lccccc@{}}
\toprule
\textbf{Training seed}
& \shortstack{\textbf{Train end / best ckpt.}\\\textbf{(k steps)}}
& \textbf{RMSE (1 / 2 / 4 s)}
& \textbf{Height MAE}
& \shortstack{\textbf{$K=8$}\\\textbf{acc. / regret}}
& \shortstack{\textbf{$K=64$}\\\textbf{acc. / regret}} \\
\midrule
0 & 30 / 30 & 0.019 / 0.040 / 0.188 & 0.0082 & 0.908 / 0.0003 & 0.574 / 0.0007 \\
1 & 13 / 8  & 0.055 / 0.116 / 0.363 & 0.0206 & 0.798 / 0.0019 & 0.398 / 0.0025 \\
2 & 11 / 6  & 0.072 / 0.125 / 0.465 & 0.0314 & 0.727 / 0.0032 & 0.355 / 0.0042 \\
3 & 30 / 28 & 0.020 / 0.040 / 0.231 & 0.0084 & 0.912 / 0.0003 & 0.575 / 0.0007 \\
4 & 11 / 6  & 0.059 / 0.114 / 0.407 & 0.0252 & 0.772 / 0.0026 & 0.369 / 0.0038 \\
\midrule
Mean-mechanism baseline
& --- & 0.436 / 0.732 / 1.082 & 0.0845 & 0.576 / 0.0111 & 0.268 / 0.0141 \\
\bottomrule
\end{tabular}%
}
\end{table}

The mean-mechanism baseline simulates the arithmetic mean of the training mechanisms $(r_m,l_1,r_l)$. For the primary fidelity and planning evaluations, predicted angles are converted to heights using the true link lengths, which isolates errors in learned dynamics from geometry estimation. Under this evaluation, the learned model substantially outperforms the mean-mechanism baseline across rollout horizons and planning resolutions (Tab.~\ref{tab:dp-wm-fidelity}). As a geometry-agnostic control, we instead convert predicted angles using the mean-mechanism link lengths for every model. At $K=8$, the learned model achieves median normalized regret $0.0095$ across training seeds, compared with $0.0226$ for the mean-mechanism baseline. Before compression, each training seed is also required to achieve both four-second height MAE and eight-candidate planning regret below 75\% of the corresponding baseline values, with stable validation rollouts. All five training seeds satisfy this fidelity criterion.

\vspace{-1mm}
\subsection{Compression, Common Scoring, and Rate Estimation}
\label{app:dp-controls}
\vspace{-1mm}

We freeze the learned physical representation and apply the same discrete compression and common-scoring framework as in App.~\ref{app:collision-model}--\ref{app:collision-compression}. For each compressed representation, a fresh common scorer is trained from simulator costs using only the decoded message, known planning initial conditions, query, and candidate velocity. This evaluates the information retained after compression separately from the rollout fidelity of the original dynamics model in App.~\ref{app:dp-representation}.

The main experiments measure the decision-supervised required rate as either candidate resolution $K$ or planning horizon $T$ changes. In the resolution sweep, the planner selects among $K\in\{2,8,64\}$ candidates, while the response target remains the same 64-action cost vector. In the horizon sweep, $T\in\{1,2,4\}$\,s varies with $K=8$. All conditions use the same frozen physical representation, so changes in required rate reflect the planning condition rather than changes in the underlying physical evidence.

Compression follows the Collision setup, with 12,000 task-only pretraining steps and 13 tolerance levels spanning $\rho=0.0005$ to $0.95$. Decision training uses unit softmax temperature and no entropy bonus. We train five seeds for each condition.

Required rates are read from the per-seed rate--regret frontiers described in App.~\ref{app:collision-frontiers}. When a target regret tolerance lies between two measured frontier points, we linearly interpolate between their coded rates. Let $R_{\mathrm{lo}}$ and $R_{\mathrm{hi}}$ denote these neighboring rates. To avoid estimating a required rate across a large unsampled region of the frontier, interpolation is omitted when both
\begin{equation}
R_{\mathrm{hi}}-R_{\mathrm{lo}}>5,
\qquad
\frac{R_{\mathrm{hi}}+0.1}{R_{\mathrm{lo}}+0.1}>3.
\end{equation}
A seed supports a target tolerance only when that tolerance is reached without crossing such a gap. For comparisons across $K$ or $T$, we use seeds supported by all three conditions and compute rate changes within seed before taking the median. At least three matched seeds are required.

The frontier also includes a zero-rate baseline that sends no learned message and always selects a fixed action chosen from the training data. Consequently, rates below one bit per context represent average coded lengths and can also arise by interpolation toward this zero-rate endpoint.

\vspace{-1mm}
\subsection{Required Rates and Input Controls}
\label{app:dp-rates}
\vspace{-1mm}

We evaluate how the required coded rate changes with candidate resolution and planning horizon. The \emph{Learned} condition uses the predictive representation from App.~\ref{app:dp-representation}. We additionally evaluate two input controls through the same downstream compression and scoring pipeline. The \emph{Known} control directly provides the standardized mechanism parameters $(\log r_m,l_1,\log r_l)$, while the \emph{History} control provides the standardized 256-dimensional calibration history without temporal-model training.

\begin{table}[h]
\centering
\small
\caption{Required coded rates as candidate resolution and planning horizon vary. Values are median required rates over matched supported seeds.}
\label{tab:dp-wm-resolution}
\vspace{1mm}
\resizebox{\linewidth}{!}{%
\begin{tabular}{@{}lccc@{}}
\toprule
\textbf{Input}
& $\boldsymbol{\delta}$
& \textbf{Resolution at $T=4$\,s: $R$ ($K=2/8/64$)}
& \textbf{Horizon at $K=8$: $R$ ($T=1/2/4$\,s)} \\
\midrule
Learned
& 0.05
& 0.000 / 0.048 / 0.059
& 0.022 / 0.029 / 0.048 \\

Learned
& 0.02
& 0.016 / 0.080 / 0.782
& 0.039 / 0.046 / 0.080 \\
\midrule
Known
& 0.05
& 0.000 / 0.024 / 0.028
& 0.026 / 0.038 / 0.024 \\

Known
& 0.02
& 0.059 / 0.055 / 0.656
& 0.047 / 0.060 / 0.061 \\

History
& 0.05
& 0.000 / 0.034 / 0.025
& 0.029 / 0.030 / 0.034 \\

History
& 0.02
& 0.022 / 0.093 / 0.929
& 0.052 / 0.105 / 0.093 \\
\bottomrule
\end{tabular}%
}
\par\vspace{2pt}
\end{table}

For the learned representation, the required rate increases consistently with both candidate resolution and planning horizon at $\delta=0.05$ and $0.02$. The paired within-seed changes have the same positive direction for both adjacent resolution and horizon comparisons. Finer action choices and longer prediction horizons therefore require more information from the learned physical representation. The Known and History rows provide corresponding input controls under the same downstream rate-estimation procedure. At $\delta=0.01$, fewer than three matched seeds reach the target tolerance across all three conditions, so we report the two tolerances with sufficient frontier support.

\textbf{Numerical accuracy.}
Simulator trajectories are generated with the classical fourth-order Runge--Kutta method (RK4) using a $0.005$\,s integration step, while the learned dynamics model operates at $0.02$\,s intervals. To verify that numerical integration error is negligible, we compare RK4 trajectories against DOP853, a higher-order adaptive Runge--Kutta solver, on 512 random mechanism--state--action combinations. The maximum height-response error is $1.98\times10^{-6}$. Sampling the RK4 trajectories on the world model's $0.02$\,s grid introduces a four-second height MAE of approximately $1.2\times10^{-4}$ relative to the finer RK4 trajectory, well below the learned-model prediction errors reported in App.~\ref{app:dp-representation}.

\vspace{-1mm}
\subsection{Autoregressive Rollout Fidelity}
\vspace{-1mm}

Fig.~\ref{fig:dp-wm-fidelity-visual} evaluates the learned dynamics model under autoregressive rollout, where each predicted state is fed back to predict the next state. Across all five training seeds, we report state RMSE as the rollout horizon increases, maximum-height prediction error, and the resulting planning regret. Fig.~\ref{fig:dp-rollout-visuals} complements these aggregate results with three predeclared test rollouts, comparing the learned dynamics with the true simulator and the mean-mechanism baseline.

\begin{figure}[h]
\centering
\includegraphics[width=\linewidth]{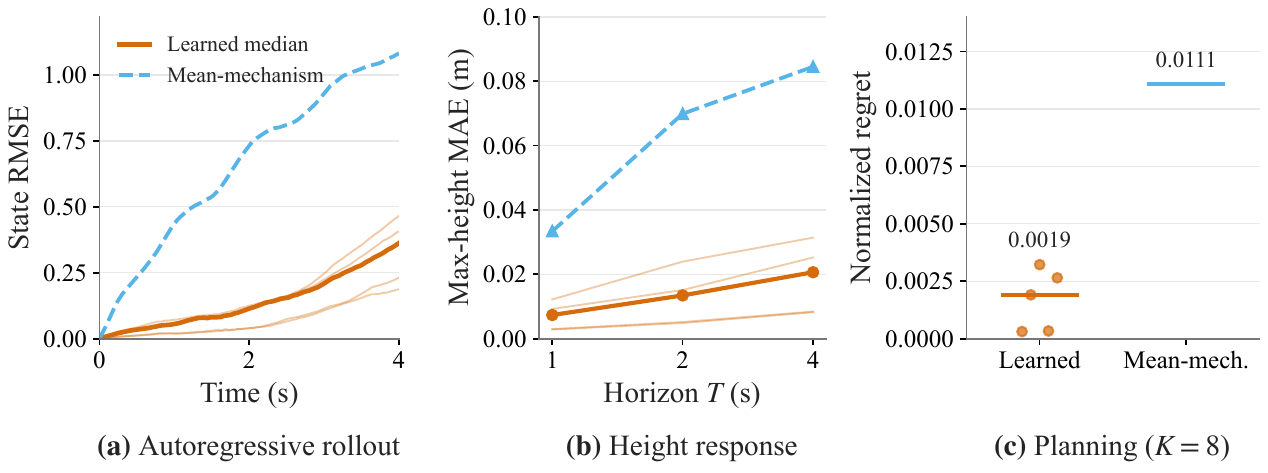}
\caption{\textbf{Autoregressive rollout fidelity.} State RMSE over the rollout horizon, maximum-height MAE, and planning regret for individual training seeds (thin orange), their medians (thick orange), and the mean-mechanism baseline (blue).}
\label{fig:dp-wm-fidelity-visual}
\end{figure}

\begin{figure}[h]
\centering
\includegraphics[width=.86\linewidth]{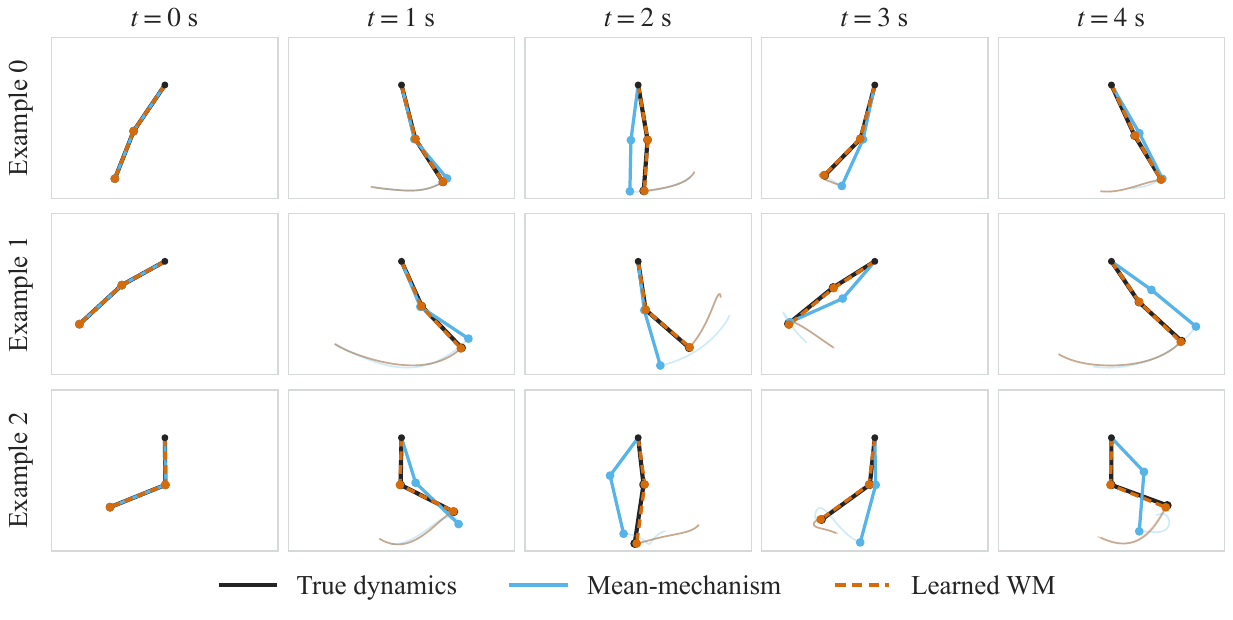}\par
\vspace{1pt}
{\footnotesize\textbf{(a)} Configuration snapshots\par}
\vspace{4pt}
\includegraphics[width=.86\linewidth]{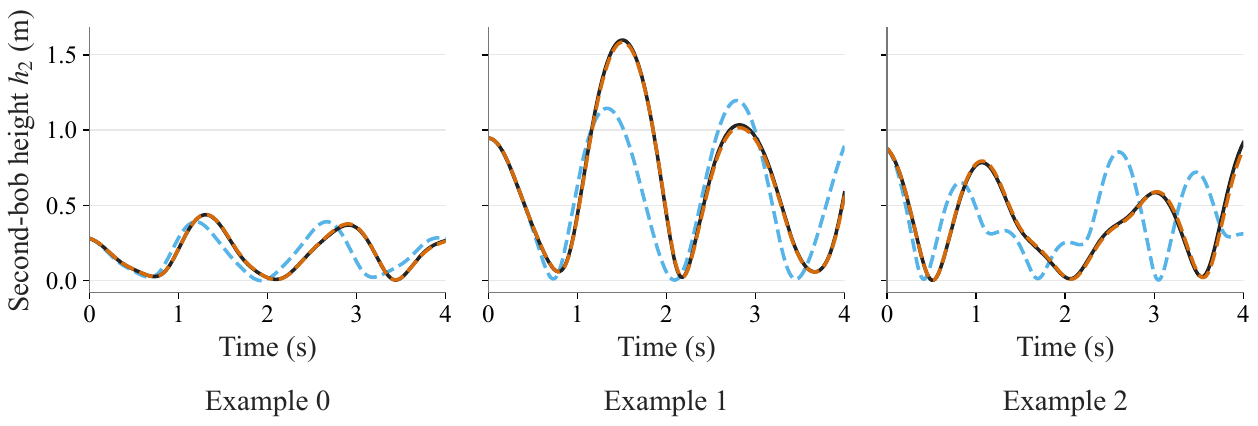}\par
\vspace{1pt}
{\footnotesize\textbf{(b)} Corresponding height trajectories\par}
\caption{\textbf{Example autoregressive rollouts.} Configuration snapshots and height trajectories for three predeclared test cases. We compare the true dynamics (black), learned model (orange dashed), and mean-mechanism baseline (blue). Faint trails show the preceding $0.6$\,s of motion.}
\label{fig:dp-rollout-visuals}
\end{figure}

\vspace{-2mm}
\section{Push-T: Candidate Sets and Planning Stages}
\label{app:pusht-search}
\vspace{-2mm}

\subsection{Visual Prediction and Query Costs}
\vspace{-1mm}

We study planning in Push-T~\citep{chi2025diffusion} using RGB observations and pusher proprioception from the DINO-WM dataset~\citep{zhou2024dino}. A frozen DINOv2 encoder~\citep{oquab2023dinov2} and pretrained JEPA-WM predictor~\citep{terver2025drives} predict future visual features for each candidate action sequence, and a frozen pose head decodes the predicted terminal block pose. Throughout this section, \emph{Full predictions} evaluate the planning query from this predicted pose, simulator states are used only for retrospective evaluation.

Both queries use the normalized terminal pose error
\begin{equation}
d_{\mathrm{pose}}
=
\frac{\|p_T-p_{\mathrm{goal}}\|_2}{s_{\mathrm{pos}}}
+
\frac{\operatorname{angdist}(\vartheta_T,\vartheta_{\mathrm{goal}})}{s_{\mathrm{yaw}}},
\end{equation}
where $\operatorname{angdist}$ is the shortest angular distance, $s_{\mathrm{pos}}=131.3$ pixels, and $s_{\mathrm{yaw}}=1.409$ rad. The fine query uses $d_{\mathrm{pose}}$ directly. The coarse query uses
$\max(0,d_{\mathrm{pose}}-0.1073)$, where $0.1073$ is the 35th percentile of training candidate distances. Selection regret is the simulator cost of the selected candidate minus the minimum simulator cost within the same candidate set, normalized by training-derived scales of $1.099$ for the fine query and $1.179$ for the coarse query.

\textbf{Frozen visual predictor.}
We use the released Push-T JEPA-WM~\citep{terver2025drives}, following its released preprocessing. The predictor receives two context frames, pusher proprioception, and candidate actions grouped into five-step blocks, and outputs future DINOv2 features.

\textbf{Pose readout.}
A separate pose head maps predicted visual features to normalized block position and sine/cosine orientation. It uses two width-256 Transformer blocks (four heads, FFN width 512) followed by an MLP. The head is trained on 240,000 frames with equally weighted position and orientation losses and validated on 4,096 frames from disjoint trajectories. AdamW uses learning rate $3\times10^{-4}$, weight decay $10^{-4}$, batch size 256, and 8,000 updates. The checkpoint with the lowest validation pose error is retained. Both the predictor and pose head remain frozen in Secs.~\ref{sec:planning} and~\ref{sec:selective_learning}.

\vspace{-1mm}
\subsection{Messages, Data, and Training}
\vspace{-1mm}

We compress each candidate's predicted terminal features into a small discrete message. The one-bit \emph{Decision} message is trained to support low-regret candidate selection, while the two- and four-bit \emph{Cost} messages are trained to preserve normalized candidate costs. After training each representation, we freeze its codes and train a fresh common scorer from the decoded message and public query information. This matched readout isolates the information retained by the message from differences between the original task-specific heads. Message sizes here are nominal bits per candidate, given by the base-two logarithm of the discrete message alphabet size.

The training bank contains 1,200 states from 1,167 trajectories, with three 16-candidate sets per state. Splits are made by trajectory, with 120 states reserved for validation and 120 for reporting, the remainder are used for training. For each candidate, the encoder reads 256 predicted visual tokens and 16 proprioceptive features using four attention slots and one width-256, four-head attention block, conditioned on the query and goal. The one-, two-, and four-bit settings select a 32-dimensional codeword from codebooks of size $2$, $4$, and $16$, respectively. VQ uses EMA codebook updates with decay $0.99$ and commitment weight $0.25$.

The common scorer uses two width-256 SiLU hidden layers. Both representation training and scorer training use AdamW with learning rate $3\times10^{-4}$, weight decay $10^{-4}$, gradient clipping at one, and cosine decay. Training budgets are summarized in Tab.~\ref{tab:pusht-message-training}. All learned modules are frozen during planning.

\begin{table}[h]
\centering
\small
\caption{Training budgets for Push-T message representations and their common scorer. Batch sizes count candidate sets.}
\label{tab:pusht-message-training}
\vspace{1mm}
\setlength{\tabcolsep}{4pt}
\begin{tabular}{@{}P{.24\linewidth}rrP{.48\linewidth}@{}}
\toprule
\textbf{Stage} & \textbf{Steps} & \textbf{Batch} & \textbf{Checkpoint rule} \\
\midrule
Encoder and native head
& 6,000 & 64
& Lowest native validation loss after a 300-step warmup. \\
\midrule
Common scorer
& 4,000 & 128
& Lowest validation decision loss after a 200-step warmup. \\
\bottomrule
\end{tabular}
\end{table}

\vspace{-1mm}
\subsection{CEM and Retrospective Evaluation}
\label{app:pusht-search-protocol}
\vspace{-1mm}

CEM~\citep{rubinstein1999cross} runs for 30 iterations with 300 candidates per iteration and ten elites, refitting a diagonal Gaussian after each update. Each candidate plan contains six world-model action steps, with five two-dimensional environment actions per step, for 30 physical actions in total. The initial normalized action distribution has zero mean and unit standard deviation.

Each CEM iteration updates the search distribution at a fixed physical state, it is not an environment interaction. In the search-versus-selection experiment, the final 300-candidate population produced by one search method can be evaluated by either selector. In the timing experiment, the planner executes the mean action sequence of the final CEM distribution. Simulator costs never guide search or selection and are used only to evaluate the resulting candidates and actions.

\vspace{-1mm}
\subsection{Dependence on the Candidate Set}
\vspace{-1mm}

To isolate the effect of the candidate set, we reuse candidate populations generated once by CEM with Full-prediction scores. For each of 128 development episodes, we save seven 300-candidate populations spanning early to late search, yielding 896 candidate sets. Every message representation is evaluated on exactly the same saved candidates.

We measure candidate-set breadth by the mean pairwise distance between simulator terminal poses, using the same position and orientation scales as the query. The regret gap between Decision and Full increases from $0.019$ on the narrowest candidate sets to $0.115$ on the broadest, showing that the information sufficient for selection depends on the alternatives being compared.

\vspace{-1mm}
\subsection{Search versus Selection and Information Timing}
\vspace{-1mm}

\textbf{Search versus selection.}
We separate the information used during CEM search from that used for final selection by crossing the searcher and selector on the same final candidate population. Under the fine query, we evaluate 64 matched development episodes using Full predictions or the one-bit Decision message during search, and then apply either selector to the resulting final 300 candidates. The mean selected-action costs for Full/Full, Full/Decision, Decision/Full, and Decision/Decision are $0.180$, $0.179$, $0.691$, and $0.732$, respectively. Replacing Full with the one-bit message only at final selection has almost no effect, whereas using the one-bit message during search substantially degrades the candidate set even when Full predictions are restored for the final choice.

\textbf{Information timing.} The timing experiment uses the same 64 matched development episodes and CEM protocol as above. The early- and late-rich schedules use identical numbers of four-bit Cost and one-bit Decision updates, differing only in their order; the two-bit Cost schedule provides a constant-information baseline with the same nominal message budget. This isolates when richer information enters the search while holding the total message budget fixed.

\vspace{-1mm}
\subsection{Cost-Message Capacity across Candidate Sets}
\vspace{-1mm}

We additionally test whether the candidate-set breadth effect persists beyond the one-bit Decision representation. Fig.~\ref{fig:planning_support_appendix} evaluates Full predictions, Decision (1 bit), and Cost messages of increasing capacity on the same 896 saved candidate sets. Selection regret decreases as the Cost message retains more information, while broader candidate sets remain more demanding across message capacities.

\begin{figure}[h]
\centering
\includegraphics[width=\linewidth]{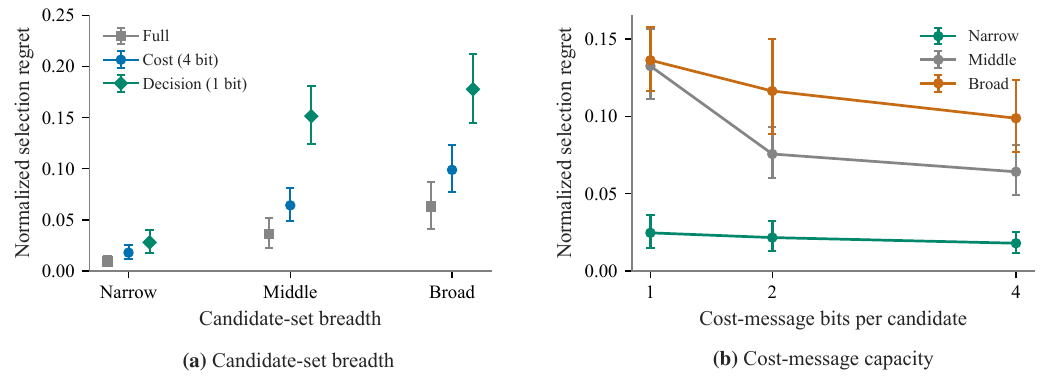}
\small
\caption{\textbf{Selection difficulty depends on the supplied candidate set.}
\textbf{(a)} Full predictions, Decision (1 bit), and Cost (4 bit) evaluated on identical candidate sets grouped by breadth. \textbf{(b)} Selection regret for Cost messages as message capacity increases.}
\label{fig:planning_support_appendix}
\end{figure}

\vspace{-2mm}
\section{PushCube: Query-Guided Candidate Proposals}
\label{app:query-search}
\vspace{-2mm}

\subsection{Environment, Observations, and Physical Data}
\vspace{-1mm}

This section details the PushCube experiment in Sec.~\ref{sec:query_search}. We use ManiSkill PushCube-v1~\citep{tao2024maniskill3} with a Panda arm under end-effector delta-position control. Each observation contains a $128\times128$ RGB image from the base camera and 16 proprioceptive features. The native goal marker is hidden from the image, the target displacement is supplied only through the planning query. We construct a query-independent physical bank with 512 training scenes, 64 validation scenes, and 192 exploratory scenes, each containing 64 executed action sequences. Queries are applied to these stored trajectories offline, so different targets reuse the same underlying physical data. The action bank combines pushes in different directions and distances with smooth perturbations to cover diverse cube motions.

Each plan contains ten macro-actions over $(\Delta x,\Delta y,\Delta z,\mathrm{gripper})$, with each macro-action repeated for four controller steps. The plans approach the cube, descend, and execute a push with the gripper closed.

\vspace{-1mm}
\subsection{Shared Predictor and Matched Proposals}
\vspace{-1mm}

A frozen DINOv2 ViT-S/14~\citep{oquab2023dinov2} encodes the RGB observation. The shared action-conditioned physical predictor combines these visual features with proprioception and a candidate action sequence, and predicts the cube and end-effector trajectory together with terminal visual features. The planning query is never provided to this predictor.

The query-aware and query-blind proposal networks use the same architecture: an eight-component diagonal Gaussian mixture over the full $10\times4$ action sequence. The query-aware proposal receives the target displacement, while the query-blind proposal receives a learned constant embedding of the same dimension, giving the two proposals matched capacity. Both are trained on the same scene--action--query tuples with cost-weighted likelihood,
\begin{equation}
w_j \propto \exp[-J_q(A_j)/\tau],
\end{equation}
where $J_q(A_j)$ evaluates the stored physical outcome under the query defined below. The temperature $\tau$ is calibrated once on the training set and then fixed for evaluation. Thus the comparison changes access to the query during candidate proposal while keeping the physical data, proposal capacity, and downstream predictor matched.

The physical predictor is trained for 12,000 updates and the two proposal networks for 8,000 updates. All models use AdamW with learning rate $3\times10^{-4}$ and batch size 128. Tab.~\ref{tab:pushcube-training} summarizes the training objectives.

\begin{table}[h]
\centering
\small
\caption{Training objectives and budgets for the PushCube proposal experiment.}
\label{tab:pushcube-training}
\vspace{1mm}
\resizebox{\linewidth}{!}{%
\begin{tabular}{lP{.58\linewidth}r}
\toprule
\textbf{Component} & \textbf{Training objective} & \textbf{Steps} \\
\midrule
Physical predictor
& Terminal-feature prediction and physical-trajectory prediction
& 12,000 \\
\midrule
Query-aware / query-blind proposals
& Cost-weighted mixture negative log likelihood
& 8,000 \\
\bottomrule
\end{tabular}%
}
\end{table}

\vspace{-1mm}
\subsection{Candidate Budgets and Execution}
\vspace{-1mm}

A query specifies a target displacement $q$ at a radius of 15\,cm from the cube's initial position. For predicted terminal displacement $\Delta p_T$, the planning cost is
\begin{equation}
\ell_q(Y,A)=\min\left(1,\frac{\|\Delta p_T-q\|_2}{0.20}\right).
\end{equation}
Planning evaluates this cost from the shared predictor, while evaluation applies the same cost to the executed displacement.

Both proposal networks use the same action-conditioned predictor to evaluate their candidates. We compare candidate budgets
$N\in\{1,4,8,16,32,64,128\}$ using nested prefixes of the same 128-sample sequence and matched proposal randomness across the two methods. The planner selects one candidate and executes its ten macro-actions open loop. Success requires the cube to finish within 3.5\,cm of the target while remaining on the table.

The primary evaluation uses 512 fresh scenes, four target directions, and three independently trained proposal pairs. A separate calibration pool fixes the comparison criteria before evaluation. With only $N=8$ candidates, the query-aware proposal achieves lower mean terminal cost and higher success than the query-blind proposal with $N=128$, while satisfying the pre-specified comparison criteria across all three training seeds. This corresponds to a $16\times$ reduction in candidate evaluations.

\vspace{-1mm}
\subsection{Candidate Quality and Selection Error}
\vspace{-1mm}

Fig.~\ref{fig:planning_query_support_diagnostics} separates the quality of the proposed candidate set from the accuracy of the learned selector. We evaluate both effects on 192 exploratory scenes. In panel (a), query-aware proposals concentrate realized cube motions more strongly around the target direction, while query-blind proposals remain more dispersed. This shows directly how access to the query focuses candidate generation.

Panel (b) tracks both the best available candidate and the candidate chosen by the learned selector as the budget increases. For the query-aware proposal, the best simulator-scored candidate continues to improve with larger $N$, indicating that additional samples uncover increasingly good actions. The selected candidate does not always improve at the same rate, showing that a better candidate set does not by itself guarantee better final selection. Simulator costs are used only for this retrospective evaluation.

Panel (c) illustrates this distinction in a single exploratory scene at $N=128$. The query-aware proposal contains a low-cost candidate, but the learned selector fails to choose it and instead selects a worse action. This example separates two sources of planning error: failing to propose a good action and failing to identify a good action once it is available.

\begin{figure}[h]
\centering
\includegraphics[width=\linewidth]{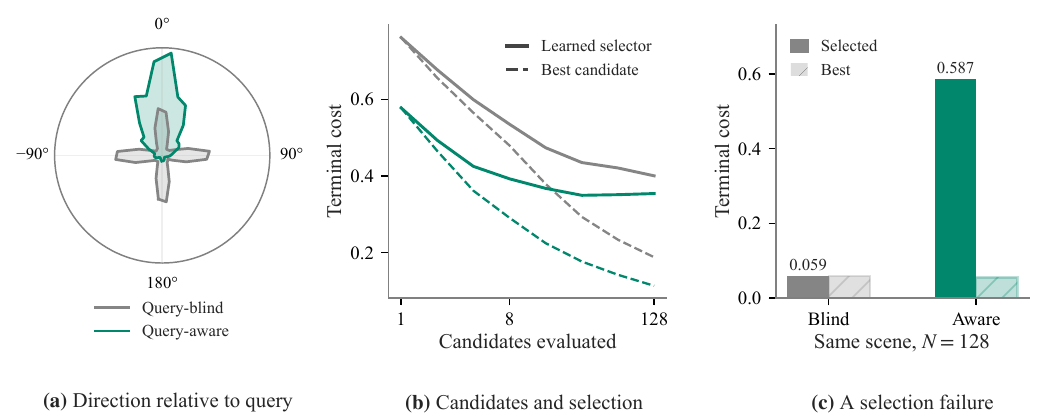}
\small
\caption{\textbf{Candidate quality and selection quality.}
\textbf{(a)} Realized candidate directions.
\textbf{(b)} Costs of the learned-selected candidate (solid) and the best candidate in each set (dashed) across candidate budgets.
\textbf{(c)} An example selection failure at $N=128$.
Gray denotes query-blind proposals and green denotes query-aware proposals.}
\label{fig:planning_query_support_diagnostics}
\end{figure}

\vspace{-2mm}
\section{Learning Planner-Relevant Representations}
\label{app:pusht-supervision}
\vspace{-2mm}

Using the frozen predictor, pose readout, and fine-query cost from App.~\ref{app:pusht-search}, we study which training signals best preserve planner-relevant information under a fixed discrete bottleneck. \emph{Value} learns to reproduce world-model costs, \emph{Planner} learns the relative distinctions used by CEM, and \emph{Hybrid} combines both objectives. All three methods use the same inputs, representation capacity, and downstream planner.

\vspace{-1mm}
\subsection{Data and Discrete Bottleneck}
\label{app:pusht-supervision-data}
\vspace{-1mm}

The training bank contains 240 contexts, with eight 64-candidate CEM populations saved from early to late search for each context. A separate 32-context pool is used for development, and confirmation uses 128 additional trajectory windows excluded from compact-model training.

For each candidate, the encoder combines predicted visual features, proprioception, and the candidate action sequence. Two permutation-equivariant Transformer blocks then process the full candidate set, allowing the representation to capture distinctions that depend on the alternatives under consideration. Each candidate is mapped to a 32-dimensional discrete code. The one-bit setting uses a two-entry codebook, while the eight-bit setting uses two 16-entry codebooks through product VQ.

All task heads receive only the decoded message together with the known query and goal. Candidate actions and continuous encoder features cannot bypass the discrete bottleneck. After representation training, we freeze the codes and train the same common readout for every method, providing a matched comparison of the information retained by each representation.

\begin{table}[h]
\centering
\small
\caption{Training budgets for Push-T supervision experiments.}
\label{tab:pusht-supervision-training}
\vspace{1mm}
\begin{tabular}{lrrl}
\toprule
\textbf{Stage} & \textbf{Steps} & \textbf{Batch} & \textbf{Training setup} \\
\midrule
Encoder and native head
& 6,000 & 64 sets
& End-to-end representation training \\
\midrule
Common readout
& 4,000 & 64 sets
& Frozen discrete codes \\
\bottomrule
\end{tabular}%
\end{table}

\vspace{-1mm}
\subsection{Value, Planner, and Hybrid Supervision}
\vspace{-1mm}

Let $c_i=\widehat J_q(A_i)/s_q$ denote the normalized cost predicted by the frozen full world model, and let $r_i\in\{1,\ldots,64\}$ be the corresponding rank within the candidate set. Value supervision preserves absolute candidate costs through
\begin{equation}
\mathcal L_{\mathrm{Value}}
=
\frac{1}{64}\sum_i
\operatorname{Huber}(\widehat c_i-c_i).
\end{equation}

Planner supervision instead emphasizes the distinctions directly used by CEM. It encourages the score $s_i$ to reflect candidate rank and separates the eight elite candidates from nearby non-elites,
\begin{align}
\mathcal L_{\mathrm{Planner}}
&=
\frac{1}{64}\sum_i
\operatorname{Huber}\!\left(
\sigma(s_i)-1+\frac{r_i-1}{63}
\right)
\nonumber\\
&\quad+
\mathbb E_{i\in E,\;j\in B}
\operatorname{softplus}(1-s_i+s_j),
\end{align}
where $E$ contains the top eight candidates and $B$ contains candidates ranked 9--24. The first term preserves the overall ordering, while the second focuses representation capacity near the CEM elite boundary. Hybrid supervision combines the two,
\begin{equation}
\mathcal L_{\mathrm{Hybrid}}
=
\mathcal L_{\mathrm{Value}}
+
\lambda\mathcal L_{\mathrm{Planner}}.
\end{equation}
We set $\lambda$ by balancing encoder gradient magnitudes on development data, yielding $\lambda=0.260$ at one bit and $\lambda=0.209$ at eight bits. These values are fixed before confirmation. All objectives are constructed from frozen world-model predictions. Simulator outcomes are used only for evaluation.

\vspace{-1mm}
\subsection{Low-Capacity Confirmation and Capacity Control}
\vspace{-1mm}

Confirmation uses 128 unseen contexts and three training seeds. At one bit, we compare Value, Planner, and Hybrid supervision. At eight bits, we compare Value and Hybrid as a higher-capacity control. CEM runs for 30 iterations with 64 candidates and eight elites.

We evaluate each representation in two ways. \emph{Elite agreement} measures whether the compressed representation identifies the same promising candidates as the full predictor, using the Jaccard overlap between their eight-candidate elite sets. A common readout is trained on each frozen representation for this comparison. \emph{Planning regret} evaluates the action ultimately produced by CEM relative to the best simulator-evaluated candidate visited by Full CEM on the same context.

\begin{table}[h]
\centering
\small
\caption{Push-T confirmation and capacity control. Lower regret and higher Jaccard are better.}
\label{tab:pusht-key-controls}
\vspace{1mm}
\begin{tabular}{cccccccccc}
\toprule
\multicolumn{6}{c}{\textbf{1 bit}} &
\multicolumn{4}{c}{\textbf{8 bit}} \\
\cmidrule(lr){1-6}\cmidrule(lr){7-10}
\multicolumn{2}{c}{\textbf{Value}} &
\multicolumn{2}{c}{\textbf{Planner}} &
\multicolumn{2}{c}{\textbf{Hybrid}} &
\multicolumn{2}{c}{\textbf{Value}} &
\multicolumn{2}{c}{\textbf{Hybrid}} \\
\cmidrule(lr){1-2}\cmidrule(lr){3-4}\cmidrule(lr){5-6}
\cmidrule(lr){7-8}\cmidrule(lr){9-10}
Regret & Jaccard &
Regret & Jaccard &
Regret & Jaccard &
Regret & Jaccard &
Regret & Jaccard \\
\midrule
0.448 & 0.082 &
0.353 & 0.124 &
0.272 & 0.136 &
0.229 & 0.161 &
0.306 & 0.142 \\
\bottomrule
\end{tabular}%
\end{table}

At one bit, Planner improves over Value on both measures, while Hybrid performs best. Under this tight bottleneck, supervision aligned with the planner therefore preserves useful distinctions more effectively than cost prediction alone.

The comparison changes at eight bits. Value achieves lower planning regret and higher elite agreement than Hybrid, indicating that the low-capacity benefit of combining planner-aware and value supervision does not persist once substantially more information can pass through the representation.

\vspace{-1mm}
\subsection{Closed-Loop Evaluation and Robustness Controls}
\label{app:pusht-closed-loop}
\vspace{-1mm}

We next evaluate whether the supervision differences persist under closed-loop replanning. Each episode contains four replanning rounds over 60 environment steps. At each round, the planner predicts a 30-step action sequence, executes the first 15 steps, observes the resulting state, and replans. Models remain frozen and planner random streams are matched across methods.

At one bit, Hybrid reduces mean closed-loop cost relative to Value by $0.113$ across three training seeds. This agrees with the candidate-set and CEM results above, showing that the benefit of planner-aware supervision carries over when prediction and planning are repeatedly updated from new observations. At eight bits, a higher-capacity diagnostic instead gives a Hybrid--Value cost difference of $+0.062$. Together with the capacity control above, this indicates that the advantage is concentrated in the low-capacity regime.

We further test whether the one-bit result depends on the amount of training data or on a particular CEM configuration. Fig.~\ref{fig:pusht_learning_controls}(a,b) summarizes the development behavior in elite agreement and planning regret. Panel (c) increases the training set from 25\% to 100\%. Additional data improves elite agreement for Planner and Hybrid, while Value changes little, so the low-capacity difference is not removed simply by providing more training contexts. Panel (d) varies the CEM candidate count over $\{32,64,128\}$ and the elite fraction over $\{1/16,1/8,1/4\}$. Hybrid achieves lower mean regret than Value in all nine tested configurations, showing that the one-bit advantage is not tied to a single CEM population size or elite fraction.

\begin{figure}[h]
\centering
\includegraphics[width=\linewidth]{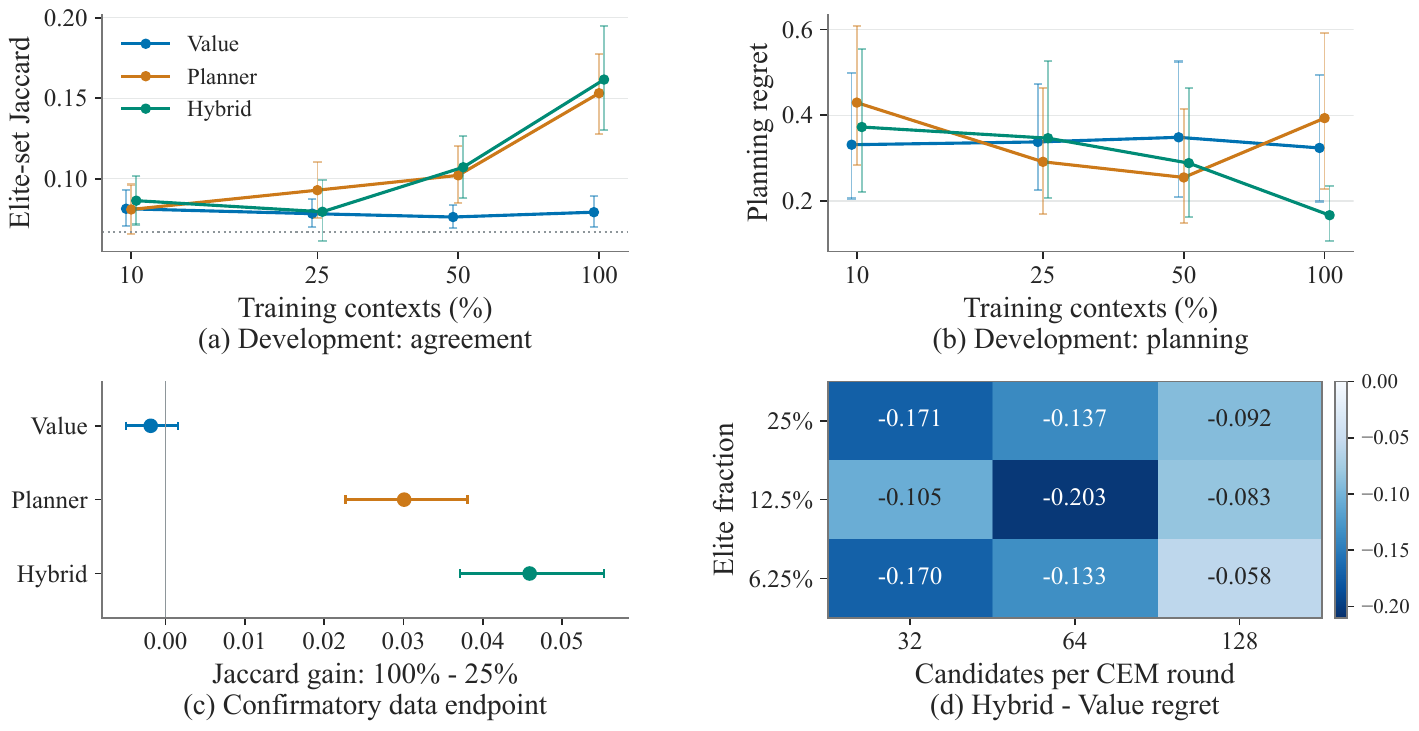}
\small
\caption{\textbf{Controls for supervision and planner settings.}
\textbf{(a,b)} Development elite agreement and planning regret.
\textbf{(c)} Effect of additional training data.
\textbf{(d)} Hybrid--Value regret across CEM candidate counts and elite fractions, where negative values favor Hybrid.}
\label{fig:pusht_learning_controls}
\end{figure}

Fig.~\ref{fig:pusht_rollout_examples} provides a qualitative view of the same capacity-dependent behavior. Panel (a) shows a one-bit episode in which Hybrid produces a lower-cost closed-loop trajectory than Value. Panel (b) shows an eight-bit episode with the opposite ordering, where Value performs better. These examples illustrate how the aggregate differences above appear in the executed trajectories rather than introducing an additional quantitative comparison.

\begin{figure}[h]
\centering
\includegraphics[width=\linewidth]{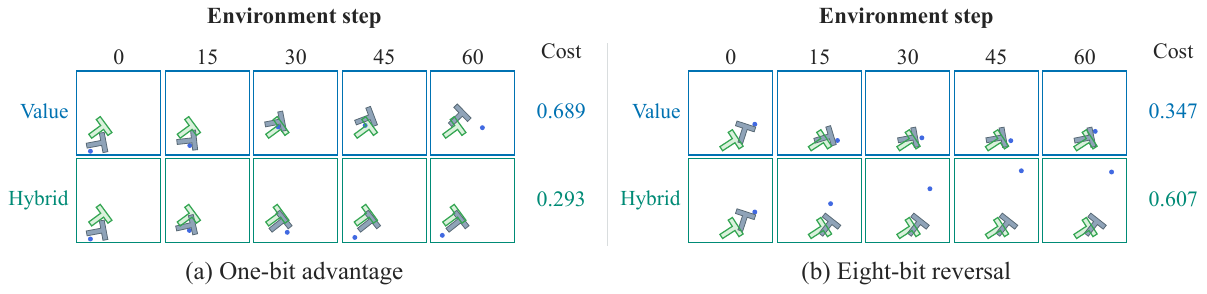}
\small
\caption{\textbf{Example closed-loop outcomes.}
Recorded trajectories showing \textbf{(a)} a one-bit Hybrid advantage and \textbf{(b)} an eight-bit Value advantage. The green outline marks the scoring target, the gray shape the block, and blue the pusher.}
\label{fig:pusht_rollout_examples}
\end{figure}

\vspace{-2mm}
\section{Query Placement across Objectives}
\label{app:oc-setting}
\vspace{-2mm}

This experiment studies where query information should enter when planning objectives change, complementing Sec.~\ref{sec:selectivity_reuse}. We use ManiSkill PushCube-v1~\citep{tao2024maniskill3} with frozen DINOv2 visual features~\citep{oquab2023dinov2}, proprioception, and ten macro-actions per plan. Objectives combine reaching with keep-out regions, speed limits, and effort constraints. These objectives change how a trajectory is evaluated without changing the underlying dynamics, so the same physical rollout can be reused across queries.

\vspace{-1mm}
\subsection{Objective Families and Physical Data}
\vspace{-1mm}

All objective families include reaching a target. Training uses five families: reach only, keep-out, speed, effort, and speed + effort. Held-out evaluation recombines these familiar objective terms into five unseen compositions: keep-out + speed, keep-out + effort, keep-out + speed + effort, two keep-out regions, and two keep-out regions + speed. Thus, the held-out setting tests generalization to new objective compositions rather than new constraint types or physical dynamics.

The physical data are shared across objectives. Each training scene contains a bank of query-independent executed plans covering straight, curved, and corrective pushes. Different queries are applied retrospectively to the same trajectories, so changing the objective does not require collecting new physical rollouts. The primary evaluation uses 256 held-out test scenes, three training seeds, and all ten objective families. Each planner proposes 32 candidates and selects one action sequence without iterative refinement.

The three primary interfaces differ in where query information enters the planning pipeline. ACWM uses a query-blind proposal and an action-conditioned physical predictor. Joint conditions both candidate generation and future prediction on the query. Modular combines a query-aware proposal with the same query-independent physical predictor used by ACWM.

\vspace{-1mm}
\subsection{Query Construction, Costs, and Success}
\label{app:oc-cost}
\vspace{-1mm}

A query specifies a target together with any active keep-out, speed, and effort constraints. Let $d$ denote terminal target distance, $v_{\max}$ the maximum cube speed along the trajectory, and $e_A$ the mean squared magnitude of the Cartesian actions. The combined planning cost is
\begin{equation}
\ell_q(Y,A)=
\frac{
c_{\mathrm{goal}}
+b_{\mathrm{keep}}c_{\mathrm{keep}}
+0.5b_{\mathrm{speed}}c_{\mathrm{speed}}
+0.5b_{\mathrm{effort}}c_{\mathrm{effort}}
}{
1+b_{\mathrm{keep}}+0.5b_{\mathrm{speed}}+0.5b_{\mathrm{effort}}
},
\end{equation}
where each $b$ indicates whether the corresponding objective is active. Writing $[x]_0^1=\min(1,\max(0,x))$, the component costs are
\begin{equation}
c_{\mathrm{goal}}
=
\left[\frac{d-\rho_q}{0.20\,\mathrm m}\right]_0^1,
\qquad
c_{\mathrm{speed}}
=
\left[\frac{v_{\max}-v_q}{s_v}\right]_0^1,
\qquad
c_{\mathrm{effort}}
=
\left[\frac{e_A-e_q}{s_e}\right]_0^1.
\end{equation}

For active keep-out regions with center $z_j$ and radius $r_j$, the penalty is
\begin{equation}
c_{\mathrm{keep}}
=
\max_j
\left[
1-\frac{\operatorname{dist}(z_j,\mathrm{path})}{r_j}
\right]_0^1.
\end{equation}
The path distance accounts for crossings between recorded macro-step positions.

Targets lie $0.14$--$0.19$\,m from the initial cube position, with goal tolerance $\rho_q\in[0.050,0.065]$\,m. Keep-out radii lie in $[0.025,0.055]$\,m. Speed and effort thresholds are calibrated from training trajectories and then fixed for evaluation, with $s_v=0.257$\,m/s and $s_e=0.169$.

A plan is successful only when it reaches the target tolerance and satisfies every active constraint, together with the table and workspace conditions. This provides a stricter complement to the continuous planning cost.

\vspace{-1mm}
\subsection{ACWM, Joint, and Modular Interfaces}
\label{app:oc-interfaces}
\vspace{-1mm}

All three interfaces begin from the same frozen visual features and proprioceptive state and use matched Gaussian-mixture proposal networks. Their main difference is where the query enters the system.

\begin{table}[h]
\centering
\small
\caption{Query placement and predictive training exposure for the three primary interfaces.}
\label{tab:oc-access}
\vspace{1mm}
\begin{tabular}{llll}
\toprule
\textbf{Interface} &
\textbf{Proposal} &
\textbf{Physical predictor} &
\textbf{Predictive exposure} \\
\midrule
ACWM &
Query-blind &
Action-conditioned, no query &
Broad actions \\
\midrule
Joint &
Query-aware &
Query-conditioned &
Query-weighted actions \\
\midrule
Modular &
Query-aware &
Action-conditioned, no query &
Broad actions \\
\bottomrule
\end{tabular}
\end{table}

ACWM predicts the future for an arbitrary supplied action sequence without receiving the query. Modular retains exactly this predictor but replaces the query-blind proposal with a query-aware one. Joint instead couples candidate generation and prediction through a query-conditioned pathway, so its future prediction also depends directly on the objective.

This distinction changes how predictions can be reused. ACWM and Modular can predict a physical outcome once and evaluate that same outcome under different objectives. Joint specializes its prediction pathway to the query used during planning.

\vspace{-1mm}
\subsection{Training and Action Exposure}
\label{app:oc-training}
\vspace{-1mm}

All interfaces are trained for 12,000 updates with matched scene and query streams. Proposal networks use cost-weighted maximum likelihood,
\begin{equation}
w_j \propto \exp[-J_q(A_j)/\tau],
\end{equation}
where the weights favor actions that perform well under the current query. The temperature is calibrated on training data and fixed before evaluation.

Physical prediction is trained from trajectory features and terminal visual features. ACWM is exposed to broad action samples independent of the query, and Modular reuses the same trained predictor. Joint instead receives query-weighted predictive examples associated with its query-conditioned proposal distribution. The comparison therefore evaluates the complete interface designs, including how query placement changes the actions encountered during predictive training.

Simulator costs are used to construct the offline proposal weights. Deployed candidate selection uses model predictions rather than simulator outcomes.

\vspace{-1mm}
\subsection{Regret Decomposition and Objective Shift}
\label{app:oc-regret}
\label{app:oc-normalization}
\vspace{-1mm}

To identify where each interface gains or loses performance, we decompose planning regret into candidate generation and final selection. For a proposed candidate set $C$, selected action $\widehat A$, and shared reference set $C_{\mathrm{ref}}$,
\begin{align}
\operatorname{Reg}_q(\widehat A;\theta,C_{\mathrm{ref}})
=
\underbrace{
\min_{A\in C}J_q(A;\theta)
-
\min_{A\in C_{\mathrm{ref}}}J_q(A;\theta)
}_{\text{candidate-set regret}} +
\underbrace{
J_q(\widehat A;\theta)
-
\min_{A\in C}J_q(A;\theta)
}_{\text{selection regret}}.
\end{align}

The first term measures whether the proposal generated good alternatives. The second measures whether the predictor selected a good action from those alternatives. We use the same simulator-evaluated reference set for all interfaces, pooling candidate populations from the proposal mechanisms together with additional goal-directed actions.

The decomposition clarifies how Modular improves over ACWM on held-out objectives. Modular reduces candidate-set regret by $0.024$, showing that query-aware proposal generation produces better alternatives. Its selection regret is $0.013$ higher, so the shared action-conditioned predictor does not always rank the improved candidate set more accurately. The improvement in candidate generation is larger, yielding a net regret reduction of $0.011$.

\begin{table}[h]
\centering
\small
\caption{Regret decomposition and objective-shift controls.}
\label{tab:oc-contrasts}
\vspace{1mm}
\begin{tabular}{lccc}
\toprule
\textbf{Comparison} & \textbf{Total regret} & \textbf{Candidate set} & \textbf{Selection} \\
\midrule
Modular $-$ ACWM, held-out
& $-0.011$ & $-0.024$ & $+0.013$ \\
\midrule
Modular $-$ Joint, seen
& $-0.014$ & -- & -- \\
\bottomrule
\end{tabular}

\vspace{1mm}

\begin{tabular}{lcc}
\toprule
& \multicolumn{2}{c}{\textbf{Joint--ACWM advantage reduction, seen $\rightarrow$ held-out}} \\
\cmidrule(lr){2-3}
& \textbf{Normalized cost} & \textbf{Raw cost} \\
\midrule
Difference
& $+0.023$ & $+0.033$ \\
\bottomrule
\end{tabular}
\end{table}

Joint improves over ACWM on seen objectives, but this advantage becomes smaller on held-out objective compositions. The same pattern remains when objective-specific normalization is removed and raw costs are used. The reduction therefore does not arise from the normalization used to combine objective families.

The seen and held-out groups contain different mixtures of constraints, so their absolute regret values should not be interpreted as a direct comparison of task difficulty. The relevant comparison here is how the relative advantage between interfaces changes when familiar objective terms are recombined.

\vspace{-1mm}
\subsection{Fixed-History Reuse and Higher-Capacity Diagnostic}
\vspace{-1mm}

Query placement also determines whether physical predictions can be reused when the objective changes. For ACWM and Modular, fixing the observation and action sequence fixes the predicted physical outcome. A cached prediction can therefore be re-scored under a new query without rerunning the physical predictor. Joint conditions its predictor directly on the query and does not expose the same query-independent prediction interface.

We additionally evaluate a \emph{Joint+action} diagnostic that augments Joint with a query-independent action-conditioned prediction route. This restores fixed-action reuse while retaining the query-conditioned pathway. Joint+action achieves lower regret than the three primary interfaces, but it also increases model capacity and broadens predictive training exposure. We therefore treat it as a diagnostic rather than a matched comparison of query placement.

\begin{table}[h]
\centering
\small
\caption{\textbf{Interface size and planning regret.} Regret is averaged within the seen and held-out objective groups.}
\label{tab:query_interfaces}
\vspace{1mm}
\resizebox{\linewidth}{!}{%
\begin{tabular}{cccccccccccc}
\toprule
\multicolumn{3}{c}{\textbf{ACWM}} &
\multicolumn{3}{c}{\textbf{Joint}} &
\multicolumn{3}{c}{\textbf{Modular}} &
\multicolumn{3}{c}{\textbf{Joint+action}} \\
\cmidrule(lr){1-3}\cmidrule(lr){4-6}\cmidrule(lr){7-9}\cmidrule(lr){10-12}
Params. (M) & Seen & Held-out &
Params. (M) & Seen & Held-out &
Params. (M) & Seen & Held-out &
Params. (M) & Seen & Held-out \\
\midrule
3.80 & 0.250 & 0.164 &
4.36 & 0.223 & 0.160 &
4.80 & 0.209 & 0.153 &
7.54 & 0.207 & 0.148 \\
\bottomrule
\end{tabular}%
}
\end{table}

The primary comparison therefore supports separating query-specific candidate generation from query-independent physical prediction. The query can guide where the planner searches, while the resulting physical predictions remain reusable across objectives. The Joint+action diagnostic shows that richer coupled architectures can improve further when given additional capacity and training exposure.

\vspace{-1mm}
\subsection{Qualitative Objective Compositions}
\vspace{-1mm}

Fig.~\ref{fig:query_visual_detail} illustrates how the three primary interfaces behave under different objective compositions from the same initial physical state. Each column after Start corresponds to a separately reset rollout. For visualization, the planner replans four times with 16 candidates per decision. The final column uses an objective composition held out during training.

The examples show how query placement changes the resulting behavior as constraints are added. They complement the aggregate results above by visualizing the trajectories produced under different objective combinations, rather than serving as an additional quantitative comparison.

\begin{figure}[h]
\centering
\includegraphics[width=\linewidth]{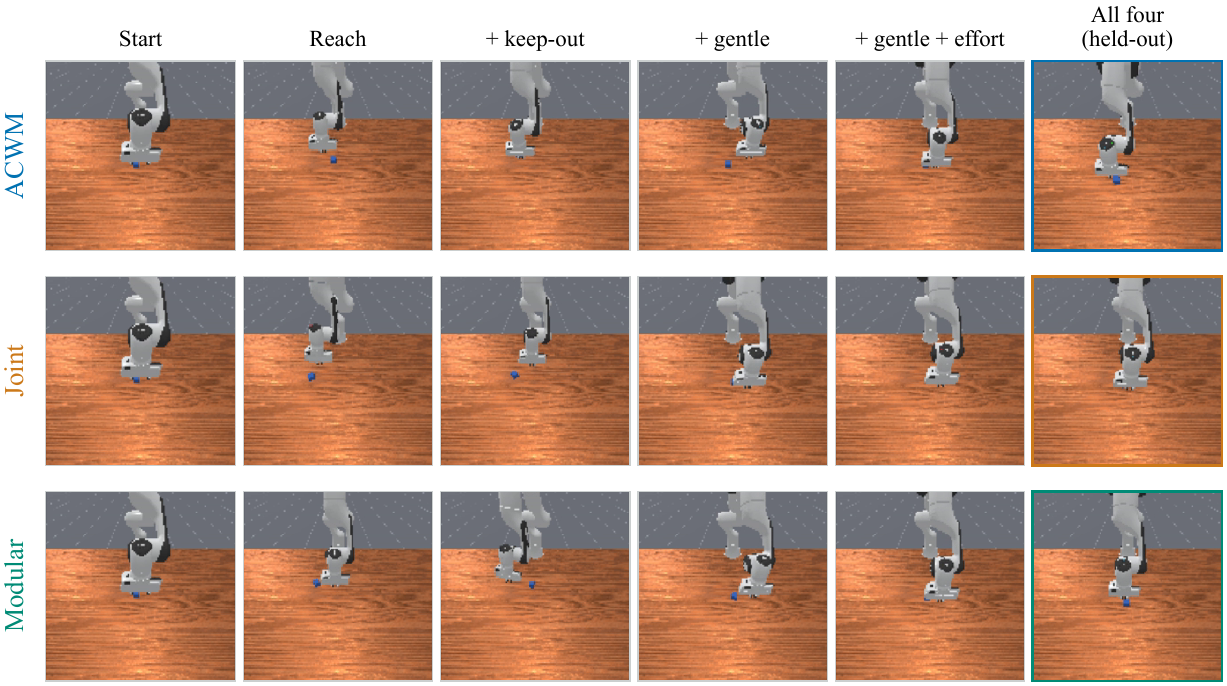}
\caption{\textbf{Three interfaces under different objective compositions.}
Each column after Start shows a separately reset rollout from the same initial physical state. The final composition was held out from training.}
\label{fig:query_visual_detail}
\end{figure}

\vspace{-2mm}
\section{Scope, Limitations, and Future Directions}
\label{app:limitations}
\vspace{-2mm}

\textbf{Scope and limitations.} This work studies what information a world model must preserve for planning, and how that requirement changes with the query, candidate set, and planner. Our formal analysis separates mechanism, response, and decision sufficiency, while the learned experiments examine how these distinctions arise in model-based planning. The empirical rates depend on the observation process, model class, optimization, coding scheme, and readout, so they characterize the studied systems rather than universal information-theoretic limits. Similarly, the continuous-action analysis isolates an idealized resolution effect that is approximated by finite candidate sets in the learned experiments. Our robotic studies use simulated Push-T and PushCube environments, which make it possible to vary candidate generation, prediction, and query information independently while keeping the underlying physics controlled. This controlled setting also defines the current scope of the empirical conclusions. Adaptive-search results are demonstrated primarily with CEM, and other planners may require different intermediate distinctions. The query-placement study changes how trajectories are evaluated while keeping the dynamics fixed, allowing prediction reuse across objectives to be measured directly. Settings in which objectives also change the available actions, observations, or interaction dynamics may introduce additional dependencies between the query and physical prediction.

\textbf{Future directions.} The results suggest several practical directions for world-model design. When objectives change frequently, separating query-guided candidate generation from action-conditioned physical prediction can allow the same predicted outcomes to be reused across tasks while directing search toward actions relevant to the current objective. This motivates planning systems that cache and reuse physical predictions, update candidate proposals as objectives change, and allocate predictive detail according to the current stage of search. A further direction is to move beyond fully query-independent or fully query-conditioned prediction by allowing only selected parts of the predictive representation to depend on the objective. Such interfaces could retain reusable physical structure while specializing the information that directly affects planning decisions. Extending these principles to longer-horizon manipulation, richer embodied tasks, and large learned world models will help determine how query-dependent planning requirements interact with model error, partial observability, and changing objectives in realistic deployment settings.

\end{document}